\documentclass{article}

\usepackage{graphicx}
\usepackage{amsmath,amssymb,amsthm}
\usepackage{algorithm}
\usepackage{algpseudocode}
\usepackage{tikz}
\usepackage{tikz-3dplot}
\usepackage{subcaption}
\usepackage{url}
\usepackage{array,tabularx}
\usepackage{multicol}
\usepackage[margin=1in]{geometry}
\usepackage{authblk}

\newcommand{\ol}[1]{\overline{#1}}
\newcommand{\Cos}{C}
\newcommand{\Sin}{S}

\theoremstyle{plain}
\newtheorem{theorem}{Theorem}[section]      

\theoremstyle{definition}

\theoremstyle{remark}

\title{Adaptive Multirobot Virtual Structure Control using Dual Quaternions}

\author{Juan I. Giribet, Alejandro S. Ghersin, Ignacio Mas, Harrison Neves Marciano, Daniel Khede Dourado Villa, M\'ario Sarcinelli-Filho}

\date{}

\begin{document}

\maketitle

\begin{abstract}
This paper presents a control strategy based on dual quaternions for the coordinated formation flying of small UAV groups. A virtual structure is employed to define the desired formation, enabling unified control of its position, orientation, and shape. This abstraction makes formation management easier by allowing a low-level controller to compute individual UAV commands efficiently. The proposed controller integrates a pose control module with a geometry-based adaptive strategy, ensuring precise and robust task execution. The effectiveness of the approach is demonstrated through both simulation and experimental results.
\end{abstract}

\section{Introduction}

Unmanned Aerial Vehicles (UAVs), particularly multi-rotor platforms, have rapidly advanced in research and applications due to their unique capabilities, including vertical takeoff and landing (VTOL), hovering, and high maneuverability. These features make them ideal for complex environments and have driven their adoption in fields such as environmental monitoring, precision agriculture, infrastructure inspection, and emergency response, among others.

A key area of recent interest is the control and coordination of multiple UAVs in formation. Formation control enables groups of UAVs to maintain specific geometric arrangements while performing tasks, offering advantages such as enhanced coverage, efficiency, and redundancy \cite{Yang2022}. These benefits are critical for applications ranging from search and rescue to cooperative tasks like cargo transport and aerial cinematography. In cargo transportation, for instance, distributing payloads across multiple UAVs allows the delivery of heavier loads \cite{Jiaming2023Multiple,Rizaldi2025MultipleUAVLoad,Bacheti2025TensionBased}, especially in remote or disaster-affected areas. In aerial filming, coordinated UAVs can capture dynamic multi-angle footage, revolutionizing creative possibilities in the entertainment industry \cite{Mademlis2023MultiUAVsFilming}.

The progress in formation control is closely tied to advancements in UAV technology, including improved autonomy, communication systems, and energy efficiency. However, challenges such as stability under external disturbances, communication delays, and the scalability of control algorithms remain areas of active research.

This paper introduces a dual quaternion-based control strategy for the formation flight of small UAV groups. The approach leverages a virtual structure to manage the formation's position and orientation, simplifying task coordination. This method is particularly effective for cooperative tasks involving small groups of UAVs. Unlike traditional approaches, this strategy treats the formation's shape variables as parameters, enabling adaptable control of the UAVs’ poses.

By abstracting formation control into a virtual structure, operators can command the position, orientation, and geometric parameters of the formation intuitively \cite{Sarcinelli-Filho2023,Lewis1997Virtual,Bu2024}. For instance, a triangular formation of three UAVs can be controlled via its center and spatial orientation, reducing operational complexity. While this architecture addresses many practical challenges, limitations such as singularities in certain formations or transitions between different geometries are considered.

The approach presented in \cite{Giribet2025} demonstrated the feasibility of controlling formations of two and three UAVs using dual quaternions, offering a compact and geometrically consistent representation. However, that formulation relied on fixed controller gains and did not consider the influence of the formation's geometry on control performance.

Recent research has further expanded on this dual-quaternion framework, addressing key limitations identified in early works. For instance, \cite{Guo2022} introduced formation control strategies for multi-robot systems under communication delays, enhancing robustness in distributed coordination. Similarly, \cite{Wu2023} proposed an adaptive control scheme based on dual quaternions capable of handling input saturation while preserving stability. The survey in \cite{Farias2023} provided a comprehensive overview of dual-quaternion control methods for multi-agent systems, highlighting their advantages in representing coupled translational and rotational dynamics. Moreover, \cite{Yuan2022} extended the application of these methods to UAVs carrying suspended loads, demonstrating their potential for cooperative aerial manipulation tasks.

The main contribution of this work consists in the introduction 
of a control law that adapts the controller gains based on the geometric configuration of the UAV formation. This adaptation is motivated by the observation that sensor noise and external disturbances affect formations differently depending on the relative distances and orientations between UAVs. To support this adaptive approach, the convergence results from the fixed-gain case are extended to accommodate continuously varying gains, ensuring stability and performance across different geometric configurations.

Additionally, a partial pose representation is introduced, allowing the controller to operate even when the global pose of the formation is not fully defined—such as the case when only relative positions or orientations are relevant to the task—without requiring structural modifications to the control law.

Simulation and experimental results are presented to validate the proposed strategy. These results demonstrate that adapting the controller gains to the formation geometry improves tracking performance, especially in scenarios with changing configurations or increased sensitivity to disturbances.

\section{DUAL QUATERNIONS}
Accurate pose representation is crucial for robots to perform complex tasks efficiently and interact intelligently with their environment. In applications such as teleoperation of UAVs and mobile robots, precise pose representations are indispensable for effective route planning and safe navigation in dynamic environments.

The Newton-Euler equations traditionally separate translational and rotational motions, leading to control laws for 3+3-DOF motion. However, dual quaternions unify these motions into a compact framework, simplifying the design of control laws for full 6-DOF motion. This unified approach is especially beneficial in underactuated systems, such as multirotors or fixed-wing aircraft, where translational and rotational motions are strongly coupled.

This method is particularly relevant for tasks requiring full 6-DOF control, such as formation flying \cite{Han2008DualQuaternion}, coordinated aerial transport, or robot docking \cite{gui2016a,gui2016b}. Dual quaternions provide a numerically stable representation of Euclidean transformations due to their smaller solution space co-dimension and direct normalization process, ensuring stability during integration.

\subsection{Quaternions}

Quaternions can be thought of as a four-dimensional generalization of the complex numbers defined as:
\begin{align*}
H := &\{ \overline{q} = q_0 + q_1i + q_2j + q_3k \mid q_i \in\mathbb{R},\\
&\, i^2 = j^2 = k^2 = ijk = -1 \},    
\end{align*}
\noindent where multiplication is non-commutative. Every quaternion $\overline{q}\in H$ has a real and imaginary part, denoted as $q_0$ and $q$, respectively.
%
Analogously to the complex numbers, its conjugate can be defined for \( \overline{q} = q_0 + q_1i + q_2j + q_3k  \):
$$
\overline{q}^* = q_0 - q_1i - q_2j - q_3k,
$$
so that (\(\overline{q} \, \overline{p})^* = \overline{p}^*  \overline{q}^*\), and its norm as
\(
\|\overline{q}\| = \sqrt{\overline{q} \, \overline{q}^*}
\).
The inverse of \( \overline{q} \neq 0 \) is $\overline{q}^{-1} = {\overline{q}^*}/{\|\overline{q}\|^2}$.

In particular, the set of unit quaternions $H_1$ is widely used in robotics being a consistent way to represent attitude. 
Unit quaternions are a Lie group under multiplication.
The inverse of multiplication is reduced to conjugation, and the unit is the real quaternion $1$. To each element in $H_1$ it is possible to assign an element in $SO(3)$:
\[
R : H_1 \to SO(3), \quad \overline{q} \mapsto R(\overline{q}) = I_3 + 2q_0 ({q}\times) + 2 ({q}\times)^2,
\]
where \( (\cdot)\times : \mathbb{R}^3 \to \text{so}(3) \) is the standard map of vectors to skew-symmetric \( 3 \times 3 \) matrices. This formula is a two-to-one map given that $R(\overline{q})=R(-\overline{q})$.

The Lie algebra of unit quaternions is the set of purely imaginary quaternions, i.e., $$h_1 = \left\{ \overline{q} \in H : \overline{q} + \overline{q}^* = 0\right\}.$$ 
Every $\overline{q}\in h_1$ is
univocally related to one element $q\in\mathbb{R}^3$. 
Given $x\in\mathbb{R}^3$, we denote $\overline{x}\in h_1$ as the quaternion with zero real part, and imaginary part equal to $x$. In a similar way, for \(\overline{q}\in h_1\), the inverse of this map is denoted as ${q}\in\mathbb{R}^3$. 

In what follows we use the notation $\Sin{(\alpha)}=\sin{(\alpha)}$ and $\Cos{(\alpha)}=\cos{(\alpha)}$.

Given $x\in\mathbb{R}^3$, it is possible to assign an element in $H_1$, as: 
\(
\exp(x) = \Cos(\|{x}\|) + \frac{\Sin(\|{x}\|)}{\|{x}\|} {x}
= \Cos\left({\theta}/{2}\right) + \Sin\left({\theta}/{2}\right) {n},
\)
where ${n}\in\mathbb{R}^3$ is a unit vector representing the axis of rotation and $\theta$ is the angle of rotation. 

For quaternions $v\in h_1$ and $\overline{q}\in H_1$, the Adjoint transformation is defined as \(Ad_{\overline{q}} \overline{x} = \overline{q} \, \overline{x} \, \overline{q}^*\).

Given the angular velocity $\overline{\omega}\in h_1$, the kinematic equation of a rigid-body is given by 
\begin{align}\label{eq:quat-kin}
\dot{\overline{q}}=\frac{1}{2} \overline{q} \, \overline{\omega} .
\end{align}

\subsection{Dual Quaternions and Unit Dual Quaternions}

Dual numbers are a generalization of real numbers \cite{clifford1873}. They are defined as:
$$
D := \left\{ x = x_P + \varepsilon x_D  \mid x_P, x_D \in \mathbb{R}, \, \varepsilon^2 = 0 \right\}, 
$$
\noindent with $x_P$ being the principal part and $x_D$ being the dual part. 
The addition and multiplication can be easily extended from real numbers, considering $\varepsilon^2=0$ and the dual numbers structure as a commutative ring. 

Dual quaternions \cite{clifford1873} can be defined as:
$$\tilde{D} = \left\{ \tilde{q} = q_P + \varepsilon q_D  \mid q_P, q_D \in H \right\}.$$
Conjugation is naturally extended for any dual quaternion \( \tilde{q} = \overline{q}_P + \varepsilon\overline{q}_D  \) as \( {\tilde{q}^*} = \overline{q}_P^* + \varepsilon\overline{q}_D^*  \). So is the norm, $\|\tilde{q}\| = \sqrt{{\tilde{q}}^* \tilde{q}}$.
Unit dual quaternions are defined as
$$
\tilde{D}_1 = \left\{ \tilde{q} \in \tilde{D} \mid \|\tilde{q}\| = 1 \right\}.
$$
They are a Lie group with the inverse of a unit dual quaternion being its conjugate. Every unit dual quaternion can be written as
\(
\tilde{q} = \overline{q} + \varepsilon\frac{1}{2} \overline{p} \, \overline{q},
\)
\noindent where $\overline{q}\in H_1$ is a unit quaternion representing a rotation and $\overline{p}\in h_1$ is purely imaginary quaternion representing a translation.

The Lie algebra of \( \tilde{D}_1 \) is the set of purely imaginary dual quaternions, i.e.,
$$
\tilde{d}_1 := \left\{ \tilde{x} \in \tilde{D} \mid {\tilde{x}^*} + \tilde{x} = 0 \right\} \cong \mathbb{R}^3 \times \mathbb{R}^3.
$$



Given the angular velocity $\omega\in\mathbb{R}^3$ and the linear velocity $v\in\mathbb{R}^3$, the attitude kinematic equation for a unit dual quaternion is given by 
\begin{align}\label{eq:quat-dual-kin}
\dot{\tilde{q}}=\frac{1}{2} \tilde{q} \, \tilde{\Omega}(\omega,v),   
\end{align}
\noindent where $\tilde{\Omega}(\omega,v)=\overline{\omega}+\varepsilon \, Ad_{\overline{q}^*} \overline{v}\in\tilde{d}_1$. 

Also given $\tilde{q}\in\tilde{D}_1$ and $\tilde{x}\in \tilde{d}_1$, the Adjoint transformation is defined as: 
$ Ad_{\tilde{q}} \tilde{x} = \tilde{q} \, \tilde{x} \, \tilde{q}^*.$
Let $\tilde{q}=\ol{q}+\frac{\varepsilon}{2}\ol{p}\,\ol{q}\in \tilde{D}_1$, and $\tilde{x}=\ol{x}_P+\varepsilon\ol{x}_D\in\tilde{d}_1$ notice that: 
$$Ad_{\tilde{q}}(\tilde{x}) = Ad_{\ol{q}}\ol{x}_P + \varepsilon(Ad_{\ol{q}} \ol{x}_D+Ad_{\ol{q}} (\ol{p}\times \ol{x}_P)).$$

\subsection{Applications of Quaternions and Dual Quaternions for Rigid Body Pose Representation}

Suppose that $b$ and $i$ represent the body (robot) and inertial frames, respectively. Let ${\omega^b} \in \mathbb{R}^3$ denote the angular velocity of the robot with respect to $i$, expressed in the body coordinates $b$, and let $\overline{q}_b^i \in H_1$ be defined as in equation \eqref{eq:quat-kin}.
Given ${p^b} \in \mathbb{R}^3$, the position in the body frame, if $\overline{p} = Ad_{\ol{q}_b^i} \overline{p^b}$, then ${p} \in \mathbb{R}^3$ is the robot's position in the inertial frame. On the other hand, $\ol{q}_i^b = {\ol{q}_b^i}^*$, and $\overline{p^b} = Ad_{{q}_i^b} \overline{p}$.

Now, suppose that $p$ is the position of the robot in the inertial frame. Then, $\tilde{q}_b^i = \overline{q}_b^i + \frac{\varepsilon}{2} \overline{p}\, \ol{q}_b^i$ satisfies the kinematic equation \eqref{eq:quat-dual-kin} with \(\tilde{\Omega}(\omega,v) = \overline{\omega}^b + \varepsilon Ad_{\overline{q}_i^b} \overline{v},\),
where $\overline{v} = \dot{\overline{p}} \in h_1$. 

From a navigation algorithms perspective, this is a practical representation, considering a strapdown configuration
where gyroscopes measure angular velocity in the body frame
while GPS measures position with respect to the ECEF (Earth-Centered Earth-Fixed) frame, which can be regarded as an inertial frame $i$ for many cases.

However, in some applications, it may be preferable to represent all quantities in inertial frames. In this case, if $\xi \in \mathbb{R}^3$ is the angular velocity that satisfies $\dot{\overline{q}_i^b} = \frac{1}{2} \overline{q}_i^b \overline{\xi}$, then $\tilde{q}_i^b = \overline{q}_i^b + \frac{\varepsilon}{2} \overline{p}^b\, \ol{q}_i^b\in \tilde{D}_1$ in this case satisfies equation \eqref{eq:quat-dual-kin} with: 
$$\tilde{\Omega}(\xi,v) = \overline{\xi} + \varepsilon Ad_{{\overline{q}_i^b}^*} \dot{\overline{p}^b} = \overline{\xi} + \varepsilon (\overline{\xi} \times \overline{p} + \overline{v}).$$

\subsection{Dual Quaternions and UAV formation control}




The properties of dual quaternions enable a compact representation of rigid-body transformations, which is crucial in scenarios that require precise spatial configurations. In the realm of UAVs, dual quaternions have gained prominence in formation control, where multiple UAVs must coordinate their movements to maintain a specified configuration. The use of dual quaternion algebra facilitates the management of relative positions and orientations among UAVs, thereby improving the efficiency of multi-UAV operations \cite{Farias2024,Kenwright2023,Qi2024}. 
This approach addresses significant challenges such as instability during leader-follower dynamics and the computational demands of real-time information sharing, enabling more effective control strategies in dynamic environments \cite{Hai2021,Yuan2024DualQC}. 
Notable applications of dual quaternions in UAV formation control include scenarios involving cable-suspended cargo, where a unified framework is used to handle both the UAV's trajectory and the load's dynamics. Simulations have demonstrated that UAV systems employing dual-quaternion-based control exhibit superior trajectory-tracking performance and stability compared to traditional methods \cite{Yuan2024DualQC}. Ongoing research continues to explore the potential of dual quaternions in enhancing UAV operations, focusing on refining control algorithms, expanding their applicability to multi-robot systems, and addressing the nuances of coupling translation and rotation. As the mathematical framework evolves, it holds promise for significant advancements in both theoretical understanding and practical applications of UAV technology \cite{Jiang2022}.

\medskip

\noindent {\bf Definition} 
Suppose that $\tilde{q}, \tilde{q}_d\in\tilde{D}_1$ are the current and desired pose of a robot, satisfying equation \eqref{eq:quat-dual-kin}, with $\tilde{\Omega}(\omega,v)$ and $\tilde{\Omega}_d(\omega,v)$, respectively. The {\it tracking error} $\tilde{\delta q}\in\tilde{D}_1$ is defined as $\tilde{\delta q}=\tilde{q}_d^*\,\tilde{q}=\overline{\delta q}+\varepsilon\frac{1}{2}\overline{\delta p}\,\overline{\delta q}$.

\medskip

In \cite{Giribet2021}, a dual quaternion-based control law for mobile robot coordination is introduced, incorporating integral action to reduce dynamic errors. Extensions to underactuated vehicle formations with efficient gain tuning are proposed in \cite{Marciano2024}. This work builds on these results, with additional insights into gain bounds.

\begin{theorem}\label{theorem_1}
Let $C\subseteq\mathbb{R}^d$ be compact. 
Let $K_{\omega,i}$, $K_{v,i}$ be self-adjoint, and $K_{\omega,p}$, $K_{v,p}$, $K_\eta$, $K_\xi:\mathbb{R}^d\to\mathbb{R}^{3\times 3}$ 
be continuous, symmetric positive definite, uniformly bounded functions on $C$;
there exist constants ${k}_{min}$, ${k}_{max}>0$ such that 
${k}_{min} I \preceq K_{\bullet}(\rho) \preceq {k}_{max} I$ for all $\rho\in C$.

Given desired angular and linear velocities $\omega_d,v_d$ and desired unit dual quaternion $\tilde{q}_d$ satisfying~\eqref{eq:quat-dual-kin}, 
suppose that $\tilde{q}$ satisfies the same kinematic equation with
\begin{align}
\overline{\omega}&=
        Ad_{\overline{\delta q}^*} \overline{\omega}_d
        - \mathrm{sign}(\delta q_0)\!\left( K_{\omega,p}(\rho)\,{\delta q} + \eta_0\,K_{\omega,i}(\rho)\,{\eta}\right), \label{eq:omegabar-final}\\
v&=  v_d - K_{v,p}(\rho)\,{\delta p} - K_{v,i}(\rho)\,{\xi}, \label{eq:vbar-final}\\
\dot{\overline{\eta}}
&= \frac{1}{2}\,\overline{\eta}\Big(\,|\delta q_0|\;K_{\omega,i}(\rho)\,{\delta q}
- \mathrm{sign}(\eta_0)\,K_\eta(\rho)\,{\eta}\Big), \label{eq:quat_PI-final}\\
\dot{{\xi}}
&= -\,K_{v,i}(\rho)\,{\delta p}+K_\xi(\rho)\,{\xi}. \label{eq:xi-final}
\end{align}

Then every solution of the closed-loop system satisfies 
$({\delta q}, {\delta p}, {\eta}, {\xi}) \to (0,0,0,0)$ a.e.
\end{theorem}

\begin{proof}
Let $\delta\tilde q:=\tilde q_d^*\tilde q$ denote the dual quaternion error, 
with principal part $\overline{\delta q}=\delta q_0+\delta q\in H_1$. 
From the quaternion kinematics~\eqref{eq:quat-kin}, 
the imaginary part satisfies
\[
\dot{\delta q} 
= \tfrac{1}{2}\big(\delta q_0 I + \delta q\times\big)\big(\,\overline{\omega}-Ad_{\overline{\delta q}^*}\overline{\omega}_d\,\big)_{\!\text{imag}},
\]
where $(\cdot)_{\text{imag}}$ denotes the imaginary part of the quaternion.
Substituting equation ~\eqref{eq:omegabar-final},

\begin{align}
\dot{\delta q} &= -\tfrac{1}{2}\,|\delta q_0|\,K_{\omega,p}(\rho)\,\delta q
- \tfrac{1}{2}\,|\delta q_0|\,\eta_0\,K_{\omega,i}(\rho)\,\eta \, -\\
&\mathrm{sign}(\delta q_0) \delta q\times \!\left( K_{\omega,p}(\rho)\,{\delta q} + \eta_0\,K_{\omega,i}(\rho)\,{\eta}\right),
\label{star}
\end{align}
and
\begin{align}
\dot{\delta q_0} = -  
\mathrm{sign}(\delta q_0) \delta q^\top \!\left( K_{\omega,p}(\rho)\,{\delta q} + \eta_0\,K_{\omega,i}(\rho)\,{\eta}\right)
\label{star2}
\end{align}

For the translational error $\delta p := p - p_d$ in the same frame as $v,v_d$, 
we have
\begin{equation}
\dot{\delta p} = v - v_d = -K_{v,p}(\rho)\,\delta p - K_{v,i}(\rho)\,\xi.
\label{starstar}
\end{equation}

Define the augmented Lyapunov function:
\[
V = \tfrac{1}{2}\|\delta q\|^2 + \tfrac{1}{2}\|\delta p\|^2 
+ \tfrac{1}{2}\|\eta\|^2 + \tfrac{1}{2}\|\xi\|^2.
\]

From equations ~\eqref{star} and~\eqref{eq:quat_PI-final} (vector part of $\dot{\overline{\eta}}$), and using that for every $z\in\mathbb{R}^3$, ${z}^T ({z}\times)=0$,
\[
\begin{aligned}
\dot V_{\text{att}}
&= \delta q^\top \dot{\delta q} + \eta^\top \dot{\eta} \\
&= -\tfrac{1}{2}|\delta q_0|\,\delta q^\top K_{\omega,p}\,\delta q
   -\tfrac{1}{2}|\delta q_0|\,\eta_0\,\delta q^\top K_{\omega,i}\,\eta\\
   &+\eta^\top\!\left(\tfrac{1}{2}\eta_0|\delta q_0|\,K_{\omega,i}\,\delta q
   - \tfrac{1}{2}|\eta_0|\,K_\eta\,\eta\right) \\
&= -\tfrac{1}{2}|\delta q_0|\,\delta q^\top K_{\omega,p}\,\delta q
   -\tfrac{1}{2}|\eta_0|\,\eta^\top K_\eta\,\eta
\end{aligned}
\]
Using the uniform bounds on
${k}_{min}I\preceq K_{\omega,p}, \, K_\eta$, there exist $c_1,c_2>0$ such that
\[
\dot V_{\text{att}} \le -c_1|\delta q_0|\|\delta q\|^2 - c_2|\eta_0|\|\eta\|^2.
\tag{A}
\]

From equations ~\eqref{starstar} and ~\eqref{eq:xi-final}:
\[
\begin{aligned}
\dot V_{\text{pos}}
&= \delta p^\top \dot{\delta p} + \xi^\top \dot{\xi} \\
&= -\,\delta p^\top K_{v,p}\,\delta p - \delta p^\top K_{v,i}\,\xi
   + \xi^\top K_{v,i}\,\delta p - \xi^\top K_\xi\,\xi \\
&= -\,\delta p^\top K_{v,p}\,\delta p  - \xi^\top K_\xi\,\xi.
\end{aligned}
\]
Using 
${k}_{min}I\preceq K_{v,p}, \, K_\xi$, 
there exist $c_3,c_4>0$ such that
\[
\dot V_{\text{pos}} \le -\,c_3\,\|\delta p\|^2 - c_4\,\|\xi\|^2.
\tag{B}
\]

Combining (A) and (B), 
\(
\dot V = \dot V_{\text{att}} + \dot V_{\text{pos}}
\le 0.
\)
Hence $V(t)$ is non-increasing, and $\delta q,\delta p,\eta,\xi$ remain bounded.

On $\dot V=0$, one must have $|\delta q_0|\delta q=0$, $\delta p=0$, $|\eta_0|\eta=0$, $\xi=0$. 
When $\delta q_0=0$, from equation ~\eqref{star2} we obtain $0=-\delta q^\top K_{\omega,p}\delta q$, 
and since $K_{\omega,p}$ is uniformly positive it follows that $\delta q = 0$. Similarly, it follows that $\eta=0$. 
Therefore, the largest invariant set contained in $\{\dot V=0\}$ is
\[
\mathcal{M} = \{\delta q=0,\ \delta p=0,\ \eta=0,\ \xi=0,\ \delta q_0,\eta_0\in\{\pm1\}\}.
\]

Since $\mathrm{sign}(\cdot)$ introduces a discontinuity, 
we interpret trajectories in the \emph{Filippov sense}. 
By the generalized LaSalle invariance principle for differential inclusions, 
every Filippov solution converges to $\mathcal{M}$ almost everywhere. 
Consequently, $\delta q\to0$, $\delta p\to0$, $\eta\to0$, $\xi\to0$. 
\end{proof}

%
%
%

The next section demonstrates how this theorem can be applied to implement a cooperative control algorithm for robot formations.

\section{CLUSTER SPACE CONTROL}
A \textit{cluster} refers to a group of robots whose states are used to compute a new aggregated \textit{cluster state}, defined in the \textit{cluster space}. Cluster-Space Control (CSC) \cite{cluster_dyn_2014} models the system as an articulated kinematic mechanism, allowing the selection of state variables for effective control and monitoring.

Formation motions are defined in \textit{cluster space}, while individual robots are ultimately commanded. Therefore, it is crucial to establish kinematic transformations that relate the \textit{cluster space} variables to those in \textit{robot space}. A \textit{cluster space} controller calculates compensation actions in \textit{cluster space} and, using these transformations, generates control commands for individual robots.

This method allows operators to specify and monitor the system’s motion from the cluster perspective, simplifying the task by abstracting control of individual robots and actuators.

Consider a system of $n$ robots, each with $m_i$ degrees of freedom. The robot state vector is ${r}_i \in \mathbb{R}^{m_i}$, and the stacked state vector in \textit{robot space} is ${r} \in \mathbb{R}^{m}$, where $m = \sum{i=1}^n m_i$. Each robot’s kinematics are described by $\dot{r}_i=f_i(r_i,u_i)$, where $u_i$ is the control command. In \textit{robot space}: 
\[\dot{r}=f(r,u)=(f_1(r_1,u_1),...,f_{n}(r_n,u_n)).\]

The \textit{cluster space} state is ${c} \in \mathbb{R}^{m}$, and its relationship with the \textit{robot space} is defined through forward and inverse kinematic transformations. The forward kinematic transformation is represented by $\psi$ with $\dot{c}=J(r)f(r,u)|_{\psi^{-1}(c)}$, where ${J}({r})$ is the Jacobian matrix.

In practice, \textit{cluster space} variables are decomposed into two components: $c_q$ for the cluster's pose (position and orientation) and $c_g$ for the cluster’s shape (geometric configuration). We assume the shape dynamics are unaffected by $c_q$, leading to: 
\begin{align} \dot{c_g}&=\Gamma (c_g,u_g),\ & \dot{c_q}=\Gamma (c_q,c_g,u_q). \end{align}

Here, $u_g$ and $u_q$ are independent control signals for geometry and pose. Given desired trajectories ${c_g}_d$ and ${c_q}_d$, the goal is to find $u_g$ and $u_q$ such that $c_g\rightarrow{c_g}_d$ and $c_q\rightarrow{c_q}_d$.

This separation allows the CSC to maintain its structure independently of the formation that is being controlled. We demonstrate this with clusters of two and three vehicles, using a dual quaternion-based controller to manage position and orientation $c_q$, while adjusting the formation geometry $c_g$.

\subsection{Three-vehicle formation }
In the case of three robots (3R), the pose of the cluster can be defined as follows. Given the positions of the robots in a local frame \(r_1, r_2, r_3 \in \mathbb{R}^3\), the center of the formation is defined as 
\(p = \frac{\sum_{i=1}^3 r_i}{3}.\) 
The cluster's orientation is well-defined by the arrangement of the robots. Let $R=[x \, y \, z]\in SO(3)$ be the rotation matrix which represents the attitude of the cluster, with $x=\frac{p-r_1}{\|p-r_1\|}$, $z_1={{(r_2-r_1)}\times(r_3-r_1)}$, $z=z_1/\|z_1\|$, $y={z} \times  x$ (see figure \ref{fig:cluster3}). 
If $\overline{q}$ is such that $R=R(\overline{q})$, then the pose of the cluster is defined by the dual quaternion $\tilde{q}=\overline{q}+\varepsilon\,\frac{1}{2}\overline{p}\,\overline{q}$. 

Regarding the geometrical parameters of the cluster, both the relative distances \(d_2\) and \(d_3\) from \(r_2\) and \(r_3\) to \(r_1\), as well as the angle \(\alpha\) between \(r_2-r_1\) and \(r_3-r_1\), are used to describe it (see figure \ref{fig:cluster3}).  Observe that
\begin{align}
r_1&=3p-r_2-r_3, \\
r_2-r_1&=d_2(\Cos{({\alpha}_{2})} x - \Sin{({\alpha}_{2})} y), \\
r_3-r_1&=d_3(\Cos{({\alpha}_{3})} x + \Sin{({\alpha}_{3})} y), 
\end{align}
\noindent where $\alpha=\alpha_2+\alpha_3$, $m^2=\frac{d_2^2+d_3^2+2d_2d_3\,\Cos{(\alpha)}}{4}$, and for the angles $\Sin{(\alpha_i)}=\frac{d_j}{2\, m}\Sin{(\alpha)}$, $\Cos{(\alpha_i)}= \frac{d_j \Cos(\alpha) + d_i}{2m}$.
Then
\[
\begin{bmatrix}
    r_1\\ r_2\\r_3
\end{bmatrix}=\begin{bmatrix}
    p-\frac{d_2\Cos{({\alpha}_{2})}+d_3\Cos{({\alpha}_{3})}}{3} x - \frac{d_3\Sin{({\alpha}_{3})}-d_2\Sin{({\alpha}_{2})}}{3} y \\
    p+\frac{2d_2\Cos{({\alpha}_{2})}-d_3\Cos{({\alpha}_{3})}}{3} x - \frac{2d_2\Sin{({\alpha}_{2})}+d_3\Sin{({\alpha}_{3})}}{3} y \\
    p+\frac{2d_3\Cos{({\alpha}_{3})}-d_2\Cos{({\alpha}_{2})}}{3} x + \frac{2d_3\Sin{({\alpha}_{3})}+d_2\Sin{({\alpha}_{2})}}{3} y 
\end{bmatrix},
\]
\normalsize
To obtain the velocity commands to control the robots, the derivatives of $r_i$ can be calculated. In order to do that, it is useful to compute the following:
\begin{align*}
\frac{d}{dt}\Sin{(\alpha_i)}&=M_{1_{ij}}\dot{d}_i+M_{2_{ij}}\dot{d}_j+M_{3_{ij}}\dot{\alpha},\\
\frac{d}{dt}\Cos{(\alpha_i)}&=N_{1_{ij}}\dot{d}_i+N_{2_{ij}}\dot{d}_j+N_{3_{ij}}\dot{\alpha},
\end{align*}
\normalsize
\noindent where
\footnotesize
\begin{align*}
M_{1_{ij}}=\frac{\partial\left(\frac{d_j}{2m}\Sin{(\alpha)}\right)}{\partial d_i}&=\frac{-d_j\Sin{(\alpha) \, (d_i+d_j\Cos{(\alpha)})}}{8m^3}\\
M_{2_{ij}}=\frac{\partial\left(\frac{d_j}{2m}\Sin{(\alpha)}\right)}{\partial d_j}&=\Sin{(\alpha)}\frac{4m^2-d_j(d_j+d_i\Cos{(\alpha)})}{8m^3}\\
M_{3_{ij}}=\frac{\partial\left(\frac{d_j}{2m}\Sin{(\alpha)}\right)}{\partial \alpha}&=d_j\frac{4m^2\Cos{(\alpha)}+d_id_j\Sin^2{(\alpha)}}{8m^3}
\end{align*}
\begin{align*}
N_{1_{ij}}=\frac{\partial\left(\frac{d_j \Cos(\alpha) + d_i}{2m}\right)}{\partial d_j}&=\frac{4m^2-(d_i+d_j\Cos{(\alpha)})^2}{8m^3}\\
N_{2_{ij}}=\frac{\partial\left(\frac{d_j \Cos(\alpha) + d_i}{2m}\right)}{\partial d_i}&=\frac{4m^2\Cos{(\alpha)} - (d_i+d_j\Cos{(\alpha)})^2}{8m^3}\\
N_{3_{ij}}=\frac{\partial\left(\frac{d_j \Cos(\alpha) + d_i}{2m}\right)}{\partial \alpha}&=d_j\Sin{(\alpha)}\frac{-4m^2+d_i(d_i+d_j\Cos{(\alpha)})}{8m^3}
\end{align*}
\normalsize

\noindent with $(i,j)=(2,3)$ and $(i,j)=(3,2)$.

Let \(\omega^i=R\omega\). It follows that \(\dot{x}={\omega^i}\times x\) and \(\dot{y}={\omega^i}\times y\). 
Letting $v = \dot{p}$, 
the relation between the formation's twist \(\tilde{\Omega}(\omega,v)\) and the robot velocities is given by:
\small


\footnotesize
\begin{align*}
3\dot{r_1} & =  3v - ( (\Cos{\alpha_2}+d_2N_{1_{23}}+d_3N_{2_{32}})x-\\
&(\Sin{\alpha_2}+d_2M_{1_{23}}-d_3M_{2_{32}})y)\dot{d}_2- \\ 
& ((\Cos{\alpha_3}+d_2N_{2_{23}}+d_3N_{1_{32}})x+(\Sin{\alpha_3}+\\
&d_3M_{1_{32}}-d_2M_{2_{23}})y)\dot{d}_3 - ((d_2N_{3_{23}}+d_3N_{3_{32}})x+\\
&(d_3M_{3_{32}}-d_2M_{3_{23}})y)\dot{\alpha} - ( d_2\Cos{\alpha_2+d_3\Cos{\alpha_3}})(\omega\times x) -\\ 
&(d_3\Sin{\alpha_3}-d_2\Sin{\alpha_2})(\omega\times y), 
\end{align*}
\begin{align*}
3\dot{r_2} &=  3v + ( (2\Cos{\alpha_2}+2d_2N_{1_{23}}-d_3N_{2_{32}})x-\\
&(2\Sin{\alpha_2}+2d_2M_{1_{23}}+d_3M_{2_{32}})y)\dot{d}_2+ \\ 
& ((-\Cos{\alpha_3}+2d_2N_{2_{23}}-d_3N_{1_{32}})x-(\Sin{\alpha_3}+\\
&d_3M_{1_{32}}+2d_2M_{2_{23}})y)\dot{d}_3 + ((2d_2N_{3_{23}}-d_3N_{3_{32}})x-\\
&(d_3M_{3_{32}}+2d_2M_{3_{23}})y)\dot{\alpha} + ( 2d_2\Cos{\alpha_2}-d_3\Cos{\alpha_3})(\omega\times x) -\\ 
&(d_3\Sin{\alpha_3}+2d_2\Sin{\alpha_2})(\omega\times y) ,
\end{align*}
\begin{align*}
3\dot{r_3} &=  3v + ( (-\Cos{\alpha_2}-d_2N_{1_{23}}+2d_3N_{2_{32}})x+\\
&(\Sin{\alpha_2}+d_2M_{1_{23}}+2d_3M_{2_{32}})y)\dot{d}_2+ \\ 
& ((2\Cos{\alpha_3}-d_2N_{2_{23}}+2d_3N_{1_{32}})x+(2\Sin{\alpha_3}+\\
&2d_3M_{1_{32}}+d_2M_{2_{23}})y)\dot{d}_3 + ((-d_2N_{3_{23}}+2d_3N_{3_{32}})x+\\
&(2d_3M_{3_{32}}+d_2M_{3_{23}})y)\dot{\alpha} + (2d_3\Cos{\alpha_3}-d_2\Cos{\alpha_2})(\omega\times x) +\\ 
&(2d_3\Sin{\alpha_3}+d_2\Sin{\alpha_2})(\omega\times y). 
\end{align*}
\normalsize
\noindent These equations allow to compute the relation between the velocity of each robot and the time derivatives of the cluster variables  
\begin{align}\label{eq:Jinverso}
\begin{bmatrix}
    \dot{r}_1, \dot{r}_2, \dot{r}_3
\end{bmatrix}=
J^{-1}(r)|_{\psi^{-1}(c)} v_c,
\end{align}
\noindent where $v_c^T=\begin{bmatrix}v^T, (\omega\times x)^T, (\omega\times y)^T, \dot{d}_2, \dot{d}_3, \dot{\alpha}\end{bmatrix}^T$. %
For the shape of the 3R formation, a simple proportional controller can be implemented as follows to track a set of prescribed geometry variables \(d_{d2}>0\), \(d_{d3}>0\) and \(\alpha_d\):
\[(\dot{d}_2,\dot{d}_3,\dot{\alpha})=(k_d(d_{d2}-d_2), k_d(d_{d3}-d_3), k_\alpha(\alpha_{d}-\alpha)).\]

\begin{figure}[t!]
    \centering
        \tdplotsetmaincoords{60}{120}
        \centering
        \resizebox{0.35\textwidth}{!}{
        \begin{tikzpicture}[tdplot_main_coords, scale=3]
            \draw[thick,->] (0,0,0) -- (1,0,0) node[anchor=north east]{$x$};
            \draw[thick,->] (0,0,0) -- (0,1,0) node[anchor=north west]{$y$};
            \draw[thick,->] (0,0,0) -- (0,0,0.6) node[anchor=south]{$z$};
            \coordinate (r1) at (3,2,2.5);
            \coordinate (r2) at (1.5,0,0.5);
            \coordinate (r3) at (1.0,1,0.5);
            \filldraw[black] (r1) circle (1pt) node[anchor=north east][above]{$r_1$};
            \filldraw[black] (r2) circle (1pt) node[anchor=north east]{$r_2$};
            \filldraw[black] (r3) circle (1pt) node[anchor=south west]{$r_3$};
            \draw[thick,dashed] (0,0,0) -- (r1) node[anchor=north east]{};
            \draw[thick,dashed] (0,0,0) -- (r2) node[anchor=north east]{};
            \draw[thick,dashed] (0,0,0) -- (r3) node[anchor=north east]{};
            \draw[-stealth] (r1) -- (r2) node[midway, above left]{$d_2$};
            \draw[-stealth] (r1) -- (r3) node[midway, above right]{$d_3$};
            \tdplotdrawarc[->]{(r1)}{0.5}{-12}{36}{anchor=south}{$\alpha$}
        \end{tikzpicture}
        }
        \caption{3R cluster geometry.}
        \label{fig:cluster3}
        \centering
        \resizebox{0.35\textwidth}{!}{
        \begin{tikzpicture}
            \draw[->] (-4,0) -- (4,0) node[right] {$y^i$};
            \draw[->] (0,0) -- (0,-5) node[right] {$x^i$};
            \draw[->] (0,0) -- (1.3,-4) node[right] {$x$};
            \draw[->] (0,0) -- (4,1.3) node[right] {$y$};

            \coordinate (cm) at (0.66,-2);
            \coordinate (r1) at (0, 0);
            \coordinate (r2) at (-2, -4);
            \coordinate (r3) at (4, {-4/sqrt(5)}); 

            \filldraw[black] (cm) circle (3pt) node[anchor=north east][right]{$p$};
            \fill (r1) circle (2pt) node[above left] {$r_1$};
            \fill (r2) circle (2pt) node[below left] {$r_2$};
            \fill (r3) circle (2pt) node[below right] {$r_3$};

            \node at (-1.5, -2) [anchor=north] {$d_2$};
            \node at (2.5, -0.5) [anchor=north] {$d_3$};

            \draw[dashed] (r1) -- (r2);
            \draw[dashed] (r1) -- (r3);
            \tdplotdrawarc[->]{(r1)}{1.2}{-75}{-25}{anchor=south,right}{$\alpha_3$}
            \tdplotdrawarc[->]{(r1)}{1.4}{-118}{-74}{anchor=north}{$\alpha_2 \, \, \,  \, \, \, $}
        \end{tikzpicture}
        }
        \caption{3R cluster geometry, pose and position.}
        \label{fig:cluster3-2D}
\end{figure}

Algorithm \ref{alg:cluster3} completes the description of the dual quaternion \textit{cluster space} controller (CSC) for the 3R formation.

\begin{algorithm}
\caption{CSC for the 3R Formation}
\begin{algorithmic}
\State \textbf{Assumptions:} Let $C\subseteq\mathbb{R}^d$ be a compact set, and let $K_{\omega,p}$, $K_{v,p}$, $K_{\omega,i}$, $K_{v,i}$, $K_{\eta}$, $K_{\xi}$  $:\mathbb{R}^d\rightarrow\mathbb{R}^{3\times 3}$ be continuous, uniformly positive-definite matrix functions on $C$. Let the gain $k_d>0$, and the control period $T_q>0$.
\State \textbf{Input:} Desired cluster attitude $\ol{q}_d\in H_1$, desired distances $d_{di}$ ($i=2,3$) between robots, desired position of the center of the formation, robot positions $r_i$ ($i=1,2,3$), and the desired heading of the formation.
\State \textbf{Output:} Commanded velocity vectors $v_{i}$ for robots $i=1,2,3$.
\State \textbf{Step:} Compute the current pose of the cluster $p=(r_1+r_2+r_3)/3$, and $R(\ol{q})\in\text{SO}(3)$, with columns $x=\frac{p-r_1}{\|p-r_1\|}$, $z_1={(r_2-r_1)\times(r_3-r_1)}$, $z=z_1/\|z_1\|$ and $y=z \times x$.
\State \textbf{Step:} Compute the attitude error $\overline{\delta q} = \ol{q}_d^* \ol{q}$.
\State \textbf{Step:} Compute the controller in equations \eqref{eq:omegabar-final} to \eqref{eq:xi-final}.
\State \textbf{Step:} Compute the $J^{-1}$ matrix of equation \eqref{eq:Jinverso}.
\end{algorithmic}\label{alg:cluster3}
\end{algorithm}

\subsection{Two-Vehicle Formation}
In the case of two robots (2R), an additional challenge arises due to the difficulty of fully defining the attitude of the cluster. Since only two angles are required to specify the orientation of the segment connecting the two robots, a complete attitude representation is unnecessary. Nevertheless, the structure of the control algorithm can be preserved by redefining the dual quaternion error to capture only the relevant angles for this particular formation.
This approach allows us to adapt the result from Theorem~\ref{theorem_1} to scenarios where the orientation of the virtual structure formed by the robot formation is not fully defined.

\begin{theorem}\label{theorem_2}
    Let $C\subseteq\mathbb{R}^d$ be a compact set and $K_{\omega,p}$, $K_{\omega,i}$, $K_{\eta}$  $:\mathbb{R}^d\rightarrow\mathbb{R}^{3\times 3}$ continuous uniformly positive definite matrices functions on $C$.
    Suppose that $\overline{q}\in H_1$ represents the attitude of the robot, and ${z_d}\in\mathbb{R}^3$ is a unit norm vector (expressed in the inertial reference frame of the problem), which represents a desired direction where the $z$-axis of the robot should be pointing, where $z=R(\overline{q})[0 \, 0 \, 1]^T$, and define 
    the tracking error as $\overline{\delta q} = \delta q_0 + {\delta q}$, 
where $\delta q_0 = \frac{1+\langle {z},{z_d}\rangle}{\|{z}+{z_d}\|}$ and ${\delta q} = \frac{{z_d}\times {z}}{\|{z}+{z_d}\|}$ (for ${z}\neq -{z_d})$. 
Suppose that $\overline{q}\in H_1$ is
given by equation~\eqref{eq:quat-kin} with:
\begin{align*}
  \overline{\omega}&=
        Ad_{\overline{\delta q}^*} \overline{\omega}_d
        -\mathrm{sign}(\delta q_0)(K_{\omega,p}(\rho)
            {\delta q} + \eta_0 K_{\omega,i}(\rho) {\eta}), \\
    \dot{\overline{\eta}}
         &= \frac{1}{2}{\overline{\eta}}
             ( - |\delta q_0| K_{\omega,i}(\rho) {\delta q}
                 + \mathrm{sign}(\eta_0)K_\eta(\rho) {\eta}
             ),
\end{align*}
\normalsize
\noindent where $\overline{\eta}\in H_1$, and $\overline{\eta}(0)=1$.
Then the $z$-axis of the robot is aligned with the desired  direction $z_d$.
\end{theorem}
\begin{proof}
 Notice that this is a particular case of the dynamics given in Theorem \ref{theorem_1}, considering only the part related to orientation. Furthermore, due to how the error is defined, it follows that $\delta {q}_0 \geq 0$.
Therefore, to complete the proof, it is necessary to show that if $\delta q \rightarrow 0$, then the axes $z$ and $z_d$ will be aligned.

Suppose that $R(\overline{q}) \in SO(3)$, where its columns represented as ${x}, {y}, {z} \in \mathbb{R}^3$ are the robot's axes, expressed in the inertial frame. Define ${n} = \frac{{z_d} \times {z}}{\|{z} \times {z_d}\|}$. To align these two vectors, a rotation about ${n}$ with an angle $\theta$ should be applied (see Fig.~\ref{fig:partial}), where ${z} \times {z_d} = \Sin(\theta) \, {n}$. 

The components of the quaternion representing this rotation are given by:
\begin{align}
\label{eq:quat_parcial}
{n} \Sin({\theta}/{2}) &= {{z_d} \times {z}}/{\|{z} + {z_d}\|}, \\
\Cos({\theta}/{2}) &= {\langle {z},{z} + {z_d} \rangle}/{\|{z} + {z_d}\|}. 
\end{align}

Thus, the quaternion representing the rotation to align ${z}$ with ${z_d}$ is given by $\overline{\delta q} = \delta q_0 + {\delta q}$, 
where $\delta q_0 = \frac{1 + \langle {z}, {z_d} \rangle}{\|{z} + {z_d}\|}$ and ${\delta q} = \frac{{z_d} \times {z}}{\|{z} + {z_d}\|}$. 
%
Observe that if ${\delta q} \rightarrow 0$, i.e., ${\delta q}_0 \rightarrow 1$, then $z \rightarrow z_d$.
\end{proof}







\begin{figure}[htbp]
    \centering
    \resizebox{0.2\textwidth}{!}{ 
    \begin{tikzpicture}
    \draw[thick, ->] (0,0) -- (-2.5,2) node[anchor=south west] {$\mathbf{z}$};
    \draw[thick, ->] (0,0) -- (2.5,1.5) node[anchor=south west] {$\mathbf{y}$};
    \draw[thick, ->] (0,0) -- (-0.8,-2.6) node[anchor=south west] {$\mathbf{x}$};
    \draw[thick, ->, red] (0,0) -- (1,3) node[anchor=north west] {$\mathbf{z_d}$};
    \draw[thick, ->, blue] (0,0) -- (0,-3) node[anchor=south east] {$\mathbf{n}$};
    \node at (0.1,-0.2) {O};
    \draw[dashed, ->] (-1,0.8) arc[start angle=160, end angle=45, radius=0.8];
    \node at (-0.3,1.7) {$\theta$};
    \end{tikzpicture}
    }
    \caption{Axis/angle representation for the 
    ${z}$ to ${z_d}$ error. }
    \label{fig:partial}
\end{figure}

Suppose that ${r_1},{r_2} \in \mathbb{R}^3$ are the positions of two robots in local coordinates. For this robot formation, the relative distance between them is used as the \textit{cluster space} variable to describe the shape defined as \(d = \|r_2 - r_1\|\). The pose of the formation is described by a unit dual quaternion \(\tilde{q} = \overline{q} + \varepsilon \frac{1}{2} \overline{p} \, \overline{q}\), where \(R(\overline{q}) \in SO(3)\) is an orthogonal matrix that satisfies \(R(\overline{q}) = [{x} \, {y} \, {z}]\), with \(z = \frac{{r_1 - r_2}}{\|{r_1 - r_2}\|}\). In other words, the third column encodes the orientation of the robot formation, and \(\overline{p} \in \mathbb{R}^3\) is such that \(\overline{p} = \frac{{r_1 + r_2}}{2}\), representing the center of the formation. 
Given a desired position for the center of the cluster and a desired orientation, Theorems \ref{theorem_1} and \ref{theorem_2} provide the appropriate control signals for the cluster.

Regarding the shape parameters of the cluster for this formation, given \(d_d > 0\), a simple proportional controller with $\dot{d} = k_d(d_d - d)$
can be again implemented to track the desired distance between the robots.

It is possible to derive the relation between the {\it cluster space} and {\it robot space} velocities to command the latter to the vehicles of the formation.
Given the pose of the cluster by the dual quaternion \(\tilde{q} = \overline{q} + \varepsilon \frac{1}{2} \overline{p} \, \overline{q}\), it follows that
\(
\dot{\tilde{q}} = \frac{1}{2} \tilde{q} \, \tilde{\Omega}(\omega, v),
\)
where \(v = \dot{p}= \frac{\dot{r}_1 + \dot{r}_2}{2}\).
Furthermore, let \(\omega^i = R(\overline{q}) \omega\), it follows that \(\dot{z} = {\omega^i} \times z\). Since \(r_2 = r_1 + d \, z\), we have:
$$\dot{r}_2 - \dot{r}_1 = \dot{d} z + d \, {\omega^i} \times z = \left(\dot{d} I + d \, {\omega^i} \times\right) \frac{r_2 - r_1}{\|r_1 - r_2\|}.$$

Therefore, the relation between the twist \(\tilde{\Omega}(\omega, v)\) and the robot velocities is given by:
\small
\begin{align}\label{eq:Jacob2R}
\begin{bmatrix}
    \dot{r}_1 \\ \dot{r}_2
\end{bmatrix}
= \begin{bmatrix}
    {I} & -{I}/{2} \\ {I} & {I}/{2}
\end{bmatrix}
\begin{bmatrix}
    v \\ \left(k_d(d_d - d) I + d \, {\omega^i} \times\right) \frac{r_2 - r_1}{\|r_1 - r_2\|}
\end{bmatrix}
\end{align}
\normalsize

Based on these results and  on Theorem \ref{theorem_2}, Algorithm \ref{alg:cluster2} controls the formation to the desired position with the desired geometry and attitude.

    


\begin{algorithm}
\caption{CSC for 2R formation}
\begin{algorithmic}
\State \textbf{Assumptions:} Let $C\subseteq\mathbb{R}^d$ be a compact set and $K_{\omega,p}$, $K_{v,p}$, $K_{\omega,i}$, $K_{v,i}$, $K_{\eta}$, $K_{\xi}$  $:\mathbb{R}^d\rightarrow\mathbb{R}^{3\times 3}$ continuous uniformly positive definite matrices functions on $C$. Let the gain $k_d>0$, and the control period $T_q>0$.
\State \textbf{Input:} Desired cluster attitude $z_d\in\mathbb{R}^3$, desired distance $d_d$ between robots, desired position of the center of the formation, robots positions $r_i$, $i=1,2$.. 
\State \textbf{Output:} Velocity commands $v_{i}$, for robots $i=1,2$.
\State \textbf{Step:} Compute the current pose of the cluster $p=(r_1+r_2)/2$, $z=(r_2-r_1)/\|r_1-r_2\|$.
\State \textbf{Step:} Compute the attitude error of the formation $\overline{\delta q} = \delta q_0 + {\delta q}$, 
with $\delta q_0 = \frac{1+\langle {z},{z_d}\rangle}{\|{z}+{z_d}\|}$ and ${\delta q} = \frac{{z_d}\times {z}}{\|{z}+{z_d}\|}$.
\State \textbf{Step:} Compute the controller in equations \eqref{eq:omegabar-final} to \eqref{eq:xi-final}.
\State \textbf{Step:} Compute the $J^{-1}$ matrix in equation \eqref{eq:Jacob2R}.
\end{algorithmic}\label{alg:cluster2}
\end{algorithm}

The results presented, based on dual quaternions for capturing the pose of a robot cluster, allow for the unification of clusters consisting of two, three, or more robots. The approach to representing the pose is similar across different cluster sizes, although the geometric parameters will vary accordingly. 

For instance, in the 3R case, Algorithm \ref{alg:cluster3} is analogous to Algorithm \ref{alg:cluster2} used in the 2R case, with adjustments made to account for the additional geometric parameters. This similarity demonstrates the flexibility of the dual quaternion representation in managing various cluster sizes while maintaining a consistent method for pose estimation and control.

\section{CONTROL ADAPTATION BASED UPON GEOMETRY}

In the context of multirobot systems, the sensors used to measure position and orientation of the robots are subject to various types of noise. This noise can induce significant variations in measurements, which in turn can affect the precision of the formation control. Since the variations caused by the noise depend on the geometric parameters of the formation, such as the distance between the robots and their relative arrangement, it becomes essential to adapt the controller gains according to these parameters.

For instance, in formations where the robots are very close to each other, even a small error in position measurements 
can have a considerable impact on variables that describe the formation orientation. In such cases, changes in orientation caused by the noise can be much more significant than in formations where the robots are more spread out. Therefore, dynamically adapting the controller gains based on the geometric characteristics of the formation becomes crucial to maintain system performance and ensure an appropriate response to 
sensor noise disturbances.

The next section shows simulation results discussing methods for adjusting the adaptive controller.

\section{SIMULATION RESULTS}
To evaluate the dual quaternion-based control strategy proposed in this work, simulations
were conducted that replicate typical scenarios requiring precise coordination among multiple UAVs. These simulations were carried out under different formation configurations and flight conditions, assessing both the UAVs' ability to maintain the formation and their adaptability to changes in geometric parameters and external disturbances.
The results of these simulations focus on key performance metrics such as stability, formation accuracy, and responsiveness to dynamic flight conditions. 

In order to test the capabilities of the adaptive controller, simulations were run which compare firstly the performance of two 2R formations. 
One of the formations uses the proposed adaptive CSC while the other uses a CSC with fixed gains. Apart from the CSCs, another controller with constant gain handles the formation's
geometry given by the distance between ``$d$'' in the 2R case.

Secondly, the performance of two 3R formations was simulated with comparison in mind as well. In the 3R case, the same CSC was employed, the difference with the 2R case laying in the computation of the dual quaternion error. The geometry control is slightly more elaborated as well as it handles $d_2$, $d_3$, $\alpha$.

\subsection{Two Robot System (2R)}

Simulation results based upon the two scenarios described in 
Fig.~\ref{fig:Maniobra} are presented.
To show the improved responses of the adaptive controllers, synthetic noise is injected on the position measurements of each individual robot. A zero mean band limited 
Gaussian noise with correlation time $t_c=2ms$ and standard deviation $\sigma=1m$ accounts for position determination errors on each of the $x$, $y$ and $z$ coordinates of each of the vehicles. It can be shown that the amplitude of the angles measurement noise for this formation is inversely proportional to the geometry parameter $d$. 

With the proposed adaptive scheme, the controller's gains can be considered as inversely proportional to the square root of the formation's inertia.
This consideration will be relevant when tackling the adaptive design for the 3R formation.

For both scenarios, comparisons were carried out between a controller with constant gains, and another one with adaptive gain scheduling (GS).
The integral $K_\omega^i$ and proportional $K_\omega^p$ gains for the latter were given by:
\begin{xalignat}{2}
    K_\omega^i (\lambda) & = k_\omega^i(\lambda)\, I_3, & K_\omega^p(\lambda) & = k_\omega^p(\lambda)\, I_3,
\end{xalignat}
with $I_3$ begin the $3\times 3$ identity matrix.
The $\lambda$ parameter is in the one dimensional unit simplex with 
 $   \lambda = {(d-d_{min})}/{(d_{max}-d_{min})}$
with $d \in \left[d_{min},d_{max} \right]$. For this problem, $d_{min}=10$ and $d_{max}=50$. The computation of the gains is performed as:
\begin{xalignat}{2}
 k_\omega^{i,p}(\lambda)  &= k^{i_1,p_1}_{\omega} \left(1-\lambda\right) + k^{i_2,p_2}_{\omega}  \lambda .
\end{xalignat}
The designed gains are listed given as follows:
$k^{p_1}_{\omega} = 10$, $k^{i_1}_{\omega} = 50 $, $k^{p_2}_{\omega}  = 60$,
$k^{i_2}_{\omega} = 300$.
%
%
The controller with fixed gains was simulated with its gains being 
an average of the gains of the adaptive controller.

\subsubsection*{2R Formation in hovering with varying $d$}

For this scenario a batch of 1000 simulation runs was completed. 
Figure~\ref{fig:Maniobra} (left) shows this maneuver. 
The trajectory prescribed for the cluster has a fixed position and a fixed orientation while the only variable that changes is $d$. 
For simplicity, to assess pointing error of the 2R formation, we describe this error in terms of spherical coordinates. With a slight abuse of jargon, we talk about \textit{azimuth} and \textit{elevation} error angles, which turn out to be intuitive to understand.

Fig.~\ref{fig:TwoRSimpleManeuverPosta} (top) 
shows the results where an estimation of the mean and standard deviation are carried out. Since this is an only hovering simulation case scenario, 
this figure shows the standard deviations of the ``\textit{azimuth}'' error, its mean being numerically close to zero in a 1000 runs batch. Overlapping in dashed blue line, the trajectory of the $d$ parameter can be seen varying from 50m to 10m. 
 Being very similar, the response of the ``\textit{elevation}'' angle has been skipped for  brevity.

The results show that for large values of $d$, the adaptive controller is more reactive.
For small values of $d$, the adaptive controller reduces its bandwidth (BW) 
showing an improved response to angle measurement noise.
Note that for small $d$, the $3\sigma$ red \textit{cloud} in the background, corresponding to the constant gains controller, has a noticeably larger amplitude than the green $3\sigma$ \textit{cloud} corresponding to the adaptive controller.

\begin{figure}[htbp]
    \centering
    \resizebox{0.4\linewidth}{!}{
\begin{tikzpicture}
	\draw[line width=1mm, ->] (0,0) -- (1,0) node[below] {{\huge $y$}};
	\draw[line width=1mm, ->] (0,0) -- (0,-1) node[below] {{\huge $x$}};
	
	\filldraw[fill=blue!20, draw=blue, thick] (-2,0) circle (0.5);
	\draw[thick, blue] (-2.35,0.35) -- (-1.65,-0.35);
	\draw[thick, blue] (-2.35,-0.35) -- (-1.65,0.35);
	\node[below, blue] at (-2,-0.6) {{\huge $r_2$}};
	
	\filldraw[fill=red!20, draw=red, thick] (2,0) circle (0.5);
	\draw[thick, red] (1.65,0.35) -- (2.35,-0.35);
	\draw[thick, red] (1.65,-0.35) -- (2.35,0.35);
	\node[below, red] at (2,-0.6) {{\huge $r_1$}};
	
	\draw[thick, dashed] (-2,0) -- (2,0);
	
	\draw[line width=1mm, ->] (0,-4) -- (1,-4) node[below] {{\huge $y$}};
	\draw[line width=1mm, ->] (0,-4) -- (0,-5) node[below] {{\huge $x$}};
	
	\filldraw[fill=blue!20, draw=blue, thick] (-5,-4) circle (0.5);
	\draw[thick, blue] (-5.35,-3.65) -- (-4.65,-4.35);
	\draw[thick, blue] (-5.35,-4.35) -- (-4.65,-3.65);
	\node[below, blue] at (-5,-4.6) {{\huge $r_2$}};
	
	\filldraw[fill=red!20, draw=red, thick] (5,-4) circle (0.5);
	\draw[thick, red] (4.65,-3.65) -- (5.35,-4.35);
	\draw[thick, red] (4.65,-4.35) -- (5.35,-3.65);
	\node[below, red] at (5,-4.6) {{\huge $r_1$}};
	
	\draw[thick, dashed] (-5,-4) -- (5,-4);
\end{tikzpicture}
    	}
    \hspace{0.05 \linewidth}  \rule{1px}{60px}
\hspace{0.05 \linewidth}
    \resizebox{0.4\linewidth}{!}{

\begin{tikzpicture}
	\draw[line width=1mm, ->] (-2.6,4) -- (-2.6,5) node[left] {{\huge $x$}};
	\draw[line width=1mm, ->] (-2.6,4) -- (-1.6,4) node[below] {{\huge $y$}};
	
	
	
	\def\RotorOneX{-8}   
	\def\RotorOneY{6.5}    
	\def\RotorTwoX{-8}    
	\def\RotorTwoY{1}    
	
	\filldraw[fill=blue!20, draw=blue, thick] (\RotorOneX,\RotorOneY) circle (0.5);
	\draw[thick, blue] (\RotorOneX-0.35,\RotorOneY+0.35) -- (\RotorOneX+0.35,\RotorOneY-0.35); 
	\draw[thick, blue] (\RotorOneX-0.35,\RotorOneY-0.35) -- (\RotorOneX+0.35,\RotorOneY+0.35); 
	
	\node [blue] at (\RotorOneX, \RotorOneY+1) {{\huge $r_2$}};
	\node [red] at (\RotorTwoX+0.5, \RotorTwoY+1) {{\huge $r_1$}};
	
	\filldraw[fill=red!20, draw=red, thick] (\RotorTwoX,\RotorTwoY) circle (0.5);
	\draw[thick, red] (\RotorTwoX-0.35,\RotorTwoY+0.35) -- (\RotorTwoX+0.35,\RotorTwoY-0.35); 
	\draw[thick, red] (\RotorTwoX-0.35,\RotorTwoY-0.35) -- (\RotorTwoX+0.35,\RotorTwoY+0.35); 
	
	\draw[thick, dashed] (\RotorOneX,\RotorOneY) -- (\RotorTwoX,\RotorTwoY);
	
	
	\def\RotorOneX{-6.5}   
	\def\RotorOneY{6}    
	\def\RotorTwoX{-5}    
	\def\RotorTwoY{2}    
	
	\node [blue] at (\RotorOneX, \RotorOneY+1) {{\huge $r_2$}};
	\node [red] at (\RotorTwoX, \RotorTwoY+1) {{\huge $r_1$}};
	
	\filldraw[fill=blue!20, draw=blue, thick] (\RotorOneX,\RotorOneY) circle (0.5);
	\draw[thick, blue] (\RotorOneX-0.35,\RotorOneY+0.35) -- (\RotorOneX+0.35,\RotorOneY-0.35); 
	\draw[thick, blue] (\RotorOneX-0.35,\RotorOneY-0.35) -- (\RotorOneX+0.35,\RotorOneY+0.35); 
	
	\filldraw[fill=red!20, draw=red, thick] (\RotorTwoX,\RotorTwoY) circle (0.5);
	\draw[thick, red] (\RotorTwoX-0.35,\RotorTwoY+0.35) -- (\RotorTwoX+0.35,\RotorTwoY-0.35); 
	\draw[thick, red] (\RotorTwoX-0.35,\RotorTwoY-0.35) -- (\RotorTwoX+0.35,\RotorTwoY+0.35); 
	
	\draw[thick, dashed] (\RotorOneX,\RotorOneY) -- (\RotorTwoX,\RotorTwoY);
	
	
	\def\RotorOneX{-4}   
	\def\RotorOneY{4}    
	\def\RotorTwoX{-1}    
	\def\RotorTwoY{4}    
	
	\node [blue] at (\RotorOneX, \RotorOneY+1) {{\huge $r_2$}};
	\node [red] at (\RotorTwoX+1, \RotorTwoY) {{\huge $r_1$}};
	
	\filldraw[fill=black!20, draw=black] (\RotorOneX+1.25,\RotorOneY+3.5) rectangle (\RotorOneX+1.5, \RotorOneY+1.5);
	\filldraw[fill=black!20, draw=black] (\RotorOneX+1.25, \RotorOneY-1) rectangle (\RotorOneX+1.5, \RotorOneY-3);
	
	\filldraw[fill=blue!20, draw=blue, thick] (\RotorOneX,\RotorOneY) circle (0.5);
	\draw[thick, blue] (\RotorOneX-0.35,\RotorOneY+0.35) -- (\RotorOneX+0.35,\RotorOneY-0.35); 
	\draw[thick, blue] (\RotorOneX-0.35,\RotorOneY-0.35) -- (\RotorOneX+0.35,\RotorOneY+0.35); 
	
	\filldraw[fill=red!20, draw=red, thick] (\RotorTwoX,\RotorTwoY) circle (0.5);
	\draw[thick, red] (\RotorTwoX-0.35,\RotorTwoY+0.35) -- (\RotorTwoX+0.35,\RotorTwoY-0.35); 
	\draw[thick, red] (\RotorTwoX-0.35,\RotorTwoY-0.35) -- (\RotorTwoX+0.35,\RotorTwoY+0.35); 
	
	\draw[thick, dashed] (\RotorOneX,\RotorOneY) -- (\RotorTwoX,\RotorTwoY);
	
	
	\def\RotorOneX{0.5}   
	\def\RotorOneY{6}    
	\def\RotorTwoX{2}    
	\def\RotorTwoY{2}    
	
	\node [blue] at (\RotorOneX, \RotorOneY+1) {{\huge $r_2$}};
	\node [red] at (\RotorTwoX, \RotorTwoY+1) {{\huge $r_1$}};
	
	\filldraw[fill=blue!20, draw=blue, thick] (\RotorOneX,\RotorOneY) circle (0.5);
	\draw[thick, blue] (\RotorOneX-0.35,\RotorOneY+0.35) -- (\RotorOneX+0.35,\RotorOneY-0.35); 
	\draw[thick, blue] (\RotorOneX-0.35,\RotorOneY-0.35) -- (\RotorOneX+0.35,\RotorOneY+0.35); 
	
	\filldraw[fill=red!20, draw=red, thick] (\RotorTwoX,\RotorTwoY) circle (0.5);
	\draw[thick, red] (\RotorTwoX-0.35,\RotorTwoY+0.35) -- (\RotorTwoX+0.35,\RotorTwoY-0.35); 
	\draw[thick, red] (\RotorTwoX-0.35,\RotorTwoY-0.35) -- (\RotorTwoX+0.35,\RotorTwoY+0.35); 
	
	\draw[thick, dashed] (\RotorOneX,\RotorOneY) -- (\RotorTwoX,\RotorTwoY);
	
	
	\def\RotorOneX{3.5}   
	\def\RotorOneY{6.5}    
	\def\RotorTwoX{3.5}    
	\def\RotorTwoY{1}    
	
	\node [blue] at (\RotorOneX, \RotorOneY+1) {{\huge $r_2$}};
	\node [red] at (\RotorTwoX+0.5, \RotorTwoY+1) {{\huge $r_1$}};
	
	\filldraw[fill=blue!20, draw=blue, thick] (\RotorOneX,\RotorOneY) circle (0.5);
	\draw[thick, blue] (\RotorOneX-0.35,\RotorOneY+0.35) -- (\RotorOneX+0.35,\RotorOneY-0.35); 
	\draw[thick, blue] (\RotorOneX-0.35,\RotorOneY-0.35) -- (\RotorOneX+0.35,\RotorOneY+0.35); 
	
	\filldraw[fill=red!20, draw=red, thick] (\RotorTwoX,\RotorTwoY) circle (0.5);
	\draw[thick, red] (\RotorTwoX-0.35,\RotorTwoY+0.35) -- (\RotorTwoX+0.35,\RotorTwoY-0.35); 
	\draw[thick, red] (\RotorTwoX-0.35,\RotorTwoY-0.35) -- (\RotorTwoX+0.35,\RotorTwoY+0.35); 
	
	\draw[thick, dashed] (\RotorOneX,\RotorOneY) -- (\RotorTwoX,\RotorTwoY);
\end{tikzpicture} 	
    	}
    \caption{Formation hovering with variable $d$ (left) and obstacle avoidance maneuver (right).}
    \label{fig:Maniobra}
\end{figure}

\begin{figure}[htbp]
    \centering
\includegraphics[width=0.5\linewidth]{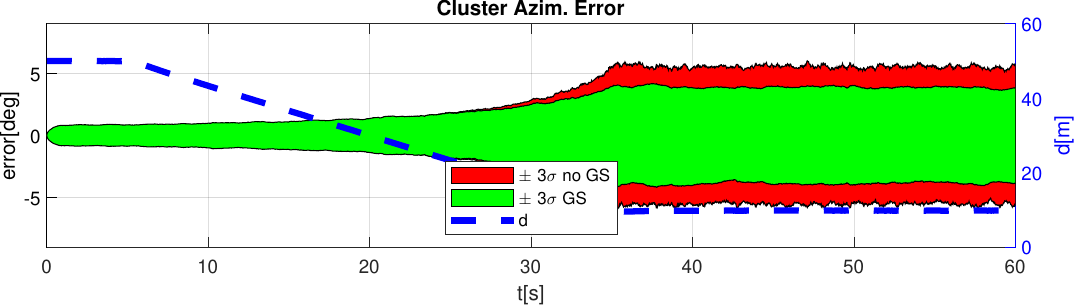}
    \includegraphics[width=0.5\linewidth]{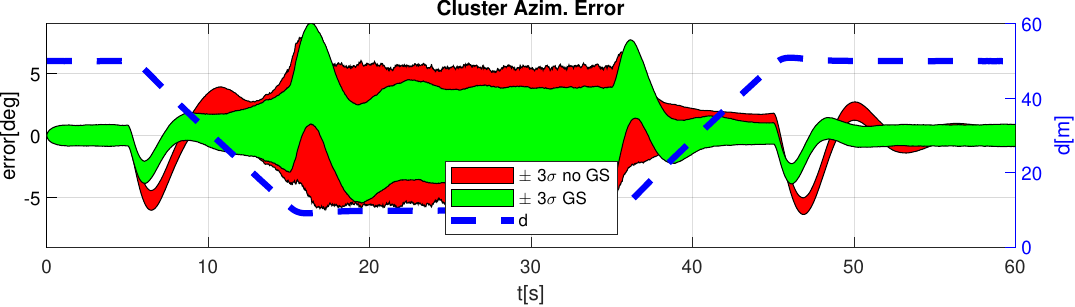}
    \caption{2R formation statistic plot. Top: Azimuth in hovering. Bottom: Azimuth in tracking.}
    \label{fig:TwoRSimpleManeuverPosta}
\end{figure}

\subsubsection*{2R Obstacle avoidance maneuver}

In Fig.~\ref{fig:Maniobra} (right), an obstacle avoidance maneuver is proposed.
The idea behind this maneuver is to pose a challenge on the controllers.
The formation has to turn while shrinking, then pass between the obstacles and finally turn again while expanding and moving towards its final position.

In this simulation, the turning while shrinking maneuver part takes place between t=5 s and t=15 s. In Fig.~\ref{fig:TwoRSimpleManeuverPosta} (bottom), notice the difference in the transient behaviors when the turning starts at t=5 s and when it ends at t=15 s. At t=5 s the higher BW of the adaptive controller renders an improved transient, while at t=15 s it renders an improved response to noise with a lower BW feedback controller. Also notice during the obstacles traversing part from t=20 s to t=35 s, the adaptive controller shows improved $\pm 3\sigma$ bounds.




\subsection{Three Robot System (3R)}

\begin{figure}[b]
 \centering
\resizebox{0.45\linewidth}{!}{
\begin{tikzpicture}
	
	\def\RotorOneX{0}   
	\def\RotorOneY{0}   
	\def\RotorTwoX{3}   
	\def\RotorTwoY{2}   
	\def\RotorThreeX{-3} 
	\def\RotorThreeY{2} 
	
	\filldraw[fill=blue!20, draw=blue, thick] (\RotorOneX,\RotorOneY) circle (0.5);
	\draw[thick, blue] (\RotorOneX-0.35,\RotorOneY+0.35) -- (\RotorOneX+0.35,\RotorOneY-0.35);  
	\draw[thick, blue] (\RotorOneX-0.35,\RotorOneY-0.35) -- (\RotorOneX+0.35,\RotorOneY+0.35);  
	\filldraw[fill=red!20, draw=red, thick] (\RotorTwoX,\RotorTwoY) circle (0.5);
	\draw[thick, red] (\RotorTwoX-0.35,\RotorTwoY+0.35) -- (\RotorTwoX+0.35,\RotorTwoY-0.35);  
	\draw[thick, red] (\RotorTwoX-0.35,\RotorTwoY-0.35) -- (\RotorTwoX+0.35,\RotorTwoY+0.35);  
	\filldraw[fill=green!20, draw=green, thick] (\RotorThreeX,\RotorThreeY) circle (0.5);
	\draw[thick, green] (\RotorThreeX-0.35,\RotorThreeY+0.35) -- (\RotorThreeX+0.35,\RotorThreeY-0.35);  
	\draw[thick, green] (\RotorThreeX-0.35,\RotorThreeY-0.35) -- (\RotorThreeX+0.35,\RotorThreeY+0.35);  
	
	\draw[thick, blue] (\RotorOneX,\RotorOneY) -- (\RotorTwoX,\RotorTwoY);
	\draw[thick, green] (\RotorOneX,\RotorOneY) -- (\RotorThreeX,\RotorThreeY);
	
	\node at (\RotorOneX,\RotorOneY-0.8) {{\huge $r_1$}};
	\node at (\RotorTwoX,\RotorTwoY+0.8) {{\huge $r_2$}};
	\node at (\RotorThreeX,\RotorThreeY+0.8) {{\huge $r_3$}};
	
	\node at (0, 1.2) {{\huge $\alpha$}};
	\draw[thick, dashed] (1.5,1) arc[start angle=40, end angle=140, radius=2];  
	
	\node at (-2, 0.7) {{\huge $d_2$}};
	\node at (2, 0.7) {{\huge $d_3$}};
	
\end{tikzpicture}

\begin{tikzpicture}
	
	\def\RotorOneX{0}   
	\def\RotorOneY{0}   
	\def\RotorTwoX{1}   
	\def\RotorTwoY{2}   
	\def\RotorThreeX{-1} 
	\def\RotorThreeY{2} 
	
	\filldraw[fill=blue!20, draw=blue, thick] (\RotorOneX,\RotorOneY) circle (0.5);
	\draw[thick, blue] (\RotorOneX-0.35,\RotorOneY+0.35) -- (\RotorOneX+0.35,\RotorOneY-0.35);  
	\draw[thick, blue] (\RotorOneX-0.35,\RotorOneY-0.35) -- (\RotorOneX+0.35,\RotorOneY+0.35);  
	\filldraw[fill=red!20, draw=red, thick] (\RotorTwoX,\RotorTwoY) circle (0.5);
	\draw[thick, red] (\RotorTwoX-0.35,\RotorTwoY+0.35) -- (\RotorTwoX+0.35,\RotorTwoY-0.35);  
	\draw[thick, red] (\RotorTwoX-0.35,\RotorTwoY-0.35) -- (\RotorTwoX+0.35,\RotorTwoY+0.35);  
	\filldraw[fill=green!20, draw=green, thick] (\RotorThreeX,\RotorThreeY) circle (0.5);
	\draw[thick, green] (\RotorThreeX-0.35,\RotorThreeY+0.35) -- (\RotorThreeX+0.35,\RotorThreeY-0.35);  
	\draw[thick, green] (\RotorThreeX-0.35,\RotorThreeY-0.35) -- (\RotorThreeX+0.35,\RotorThreeY+0.35);  
	
	\draw[thick, blue] (\RotorOneX,\RotorOneY) -- (\RotorTwoX,\RotorTwoY);
	\draw[thick, green] (\RotorOneX,\RotorOneY) -- (\RotorThreeX,\RotorThreeY);
	
	\node at (\RotorOneX,\RotorOneY-0.8) {{\huge $r_1$}};
	\node at (\RotorTwoX,\RotorTwoY+0.8) {{\huge $r_2$}};
	\node at (\RotorThreeX,\RotorThreeY+0.8) {{\huge $r_3$}};
	
	\node at (0, 1.2) {{\huge $\alpha$}};
	\draw[thick, dashed] (0.7,1.5) arc[start angle=70, end angle=110, radius=2];  
	
	\node at (-1, 0.7) {{\huge $d_2$}};
	\node at (1, 0.7) {{\huge $d_3$}};
\end{tikzpicture}
		}
\hspace{0.05 \linewidth}  \rule{1px}{60px}
\hspace{0.05 \linewidth}
    \resizebox{0.4\linewidth}{!}{
\begin{tikzpicture}
	
	\def\RotorOneX{-3}   
	\def\RotorOneY{2}   
	
	\def\RotorTwoX{1}   
	\def\RotorTwoY{1}   
	
	\def\RotorThreeX{1} 
	\def\RotorThreeY{3} 
	
	\filldraw[fill=blue!20, draw=blue, thick] (\RotorOneX,\RotorOneY) circle (0.5);
	\draw[thick, blue] (\RotorOneX-0.35,\RotorOneY+0.35) -- (\RotorOneX+0.35,\RotorOneY-0.35);  
	\draw[thick, blue] (\RotorOneX-0.35,\RotorOneY-0.35) -- (\RotorOneX+0.35,\RotorOneY+0.35);  
	\filldraw[fill=red!20, draw=red, thick] (\RotorTwoX,\RotorTwoY) circle (0.5);
	\draw[thick, red] (\RotorTwoX-0.35,\RotorTwoY+0.35) -- (\RotorTwoX+0.35,\RotorTwoY-0.35);  
	\draw[thick, red] (\RotorTwoX-0.35,\RotorTwoY-0.35) -- (\RotorTwoX+0.35,\RotorTwoY+0.35);  
	\filldraw[fill=green!20, draw=green, thick] (\RotorThreeX,\RotorThreeY) circle (0.5);
	\draw[thick, green] (\RotorThreeX-0.35,\RotorThreeY+0.35) -- (\RotorThreeX+0.35,\RotorThreeY-0.35);  
	\draw[thick, green] (\RotorThreeX-0.35,\RotorThreeY-0.35) -- (\RotorThreeX+0.35,\RotorThreeY+0.35);  
	
	\draw[thick, blue] (\RotorOneX,\RotorOneY) -- (\RotorTwoX,\RotorTwoY);
	\draw[thick, green] (\RotorOneX,\RotorOneY) -- (\RotorThreeX,\RotorThreeY);
	
	\node at (\RotorOneX,\RotorOneY-0.8) {{\large $r_1$}};
	\node at (\RotorTwoX,\RotorTwoY+0.8) {{\large $r_2$}};
	\node at (\RotorThreeX,\RotorThreeY+0.8) {{\large $r_3$}};
	
	\node at (-1.3, 2) {{\large $\alpha$}};
	\draw[thick, dashed] (-1.55,1.63) arc[start angle=-10, end angle=10, radius=2];  
	
	\node at (-1, 1) {{\large $d_2$}};
	\node at (-1, 3) {{\large $d_3$}};
	
\end{tikzpicture}

\begin{tikzpicture}
	
	\def\RotorOneX{-1}   
	\def\RotorOneY{2}   
	
	\def\RotorTwoX{-0.5}   
	\def\RotorTwoY{0}   
	
	\def\RotorThreeX{-0.5} 
	\def\RotorThreeY{4} 
	
	\filldraw[fill=blue!20, draw=blue, thick] (\RotorOneX,\RotorOneY) circle (0.5);
	\draw[thick, blue] (\RotorOneX-0.35,\RotorOneY+0.35) -- (\RotorOneX+0.35,\RotorOneY-0.35);  
	\draw[thick, blue] (\RotorOneX-0.35,\RotorOneY-0.35) -- (\RotorOneX+0.35,\RotorOneY+0.35);  
	\filldraw[fill=red!20, draw=red, thick] (\RotorTwoX,\RotorTwoY) circle (0.5);
	\draw[thick, red] (\RotorTwoX-0.35,\RotorTwoY+0.35) -- (\RotorTwoX+0.35,\RotorTwoY-0.35);  
	\draw[thick, red] (\RotorTwoX-0.35,\RotorTwoY-0.35) -- (\RotorTwoX+0.35,\RotorTwoY+0.35);  
	\filldraw[fill=green!20, draw=green, thick] (\RotorThreeX,\RotorThreeY) circle (0.5);
	\draw[thick, green] (\RotorThreeX-0.35,\RotorThreeY+0.35) -- (\RotorThreeX+0.35,\RotorThreeY-0.35);  
	\draw[thick, green] (\RotorThreeX-0.35,\RotorThreeY-0.35) -- (\RotorThreeX+0.35,\RotorThreeY+0.35);  
	
	\draw[thick, blue] (\RotorOneX,\RotorOneY) -- (\RotorTwoX,\RotorTwoY);
	\draw[thick, green] (\RotorOneX,\RotorOneY) -- (\RotorThreeX,\RotorThreeY);
	
	\node at (\RotorOneX,\RotorOneY+0.7) {{\large $r_1$}};
	\node at (\RotorTwoX+0.8,\RotorTwoY) {{\large $r_2$}};
	\node at (\RotorThreeX,\RotorThreeY+0.8) {{\large $r_3$}};
	
	\node at (-0.1, 2) {{\large $\alpha$}};
	\draw[thick, dashed] (-0.8,1.25) arc[start angle=-60, end angle=60, radius=1];  
	
	\node at (-0.5, 1) {{\large $d_2$}};
	\node at (-0.5, 3.25) {{\large $d_3$}};
	
\end{tikzpicture}
    	}
    \caption{Variable geometries emphasizing
    characteristics for: roll/yaw GS (left) and pitch GS (right).}
    \label{fig:PitchManeuver3R}
\end{figure}

In the case of the 3R cluster, six different scenarios where simulated, three of them in  hovering, and three other performing  maneuvers. All simulations were carried out with geometry transitions which are described in 
Fig.~\ref{fig:PitchManeuver3R}.

To show the improved characteristics of the adaptive controller, simulated noise was injected on the position measurements of each individual robot in the same way as for the 2R case.
For all scenarios, comparisons were carried out between an adaptive controller and another one with constant gains.
The integral $K_\omega^i$ and proportional $K_\omega^p$ gains for the former were given by:
\begin{align}
 K_\omega^i &=  diag(
 k_{\omega_x}^i(\lambda_x), k_{\omega_y}^i(\lambda_y), k_{\omega_z}^i(\lambda_z)),\\
 K_\omega^p &= diag(k_{\omega_x}^p(\lambda_x),k_{\omega_y}^p(\lambda_y),k_{\omega_z}^p(\lambda_z)).
\end{align}
The $\lambda_{x,y,z}$ parameters are in the one dimensional unit simplex with
\begin{align}
    \lambda_{x,y,z} = 
{\left(\sqrt{I_{_{x,y,z}}}-\sqrt{I_{x,y,z}^{min}}\right)}/{\left(
    \sqrt{I_{x,y,z}^{max}}-\sqrt{I_{x,y,z}^{min}}\right)}. \label{eq:sqrtGS}
\end{align}
The computation of the gains is performed as:
\begin{align}
k_{\omega_{x,y,z}}^i(\lambda) & = k^{i_1}_{\omega_{x,y,z}} \left(1-\lambda_{x,y,z} \right) + k^{i_2}_{\omega_{x,y,z}}  \lambda_{x,y,z},  \\
k_{\omega_{x,y,z}}^p(\lambda) & = k^{p_1}_{\omega_{x,y,z}} \left(1-\lambda_{x,y,z}\right) + k^{p_2}_{\omega_{x,y,z}}  \lambda_{x,y,z}.
\end{align}
For all axes, 
$k_{\omega_{x,y,z}}^p(\lambda) = {k_{\omega_{x,y,z}}^i(\lambda)}/{2}$, 
with the designed integral gains given as follows: 
$k^{i_1}_{\omega_x} = 0.5$, $k^{i_1}_{\omega_y} = 0.32$, $k^{i_1}_{\omega_z} = 0.08 $
$k^{i_2}_{\omega_x} = 2.5$, $k^{i_2}_{\omega_y} = 3.2$, $k^{i_2}_{\omega_z} = 0.8$.

The rationale behind the proposed \eqref{eq:sqrtGS} GS strategy is as follows. In the 2R case, a simple small angles trigonometry argument allows for understanding that the constant power of measurement disturbances in \textit{robot space} translates into a measurement disturbance in \textit{cluster space} whose magnitude changes as the formation changes its geometry. Namely, in the 2R case, disturbance amplitudes are inversely proportional to the $d$ parameter. In the 2R case the $d$ parameter is proportional the square root of the formation's inertia. Results confirm the main idea used for GS in this work which is: the higher the square root of the Inertia, the higher the attitude controllers' gains. Extending this idea to the 3R case, GS takes place based upon the formula in Eq.~\eqref{eq:sqrtGS}.

In the examples shown below, $d_{2,3}^{max}=50m$ and $d_{2,3}^{min}=20m$. In turn, $\alpha_{max}=150^\circ$ and $\alpha_{min}=30^\circ$. These bounds on the geometry parameters allow for a minimum distance of $10m$ between robots 2 and 3, a lower bound compatible with the $\sigma = 1m$ noise in a practical case.

The geometry transitions proposed for simulation in this work are shown in 
Fig.~\ref{fig:PitchManeuver3R}. They have the idea of going from high gains to low gains for all the axes of the cluster's attitude controller. Because of the particular geometry of the 3R cluster, the geometry transition of 
Fig.~\ref{fig:PitchManeuver3R} (left), renders a variation from maximum to minimum possible cluster inertia in the roll and yaw axes while the transition of Fig.~\ref{fig:PitchManeuver3R} 
(right) renders a variation from maximum to minimum possible cluster inertia in the pitch axis.
For comparing, a controller with fixed gains was simulated operating on a formation in parallel with the formation being controlled by the adaptive controller.
The fixed gains of this controller were set to be an average of the gains of the adaptive controller.
%



\subsubsection*{3R Formations in hovering}

For this scenario, a batch of 1000 simulations each one taking 180 seconds was completed with the geometry transition depicted in Fig.~\ref{fig:PitchManeuver3R} (left).
The trajectory prescribed for the cluster prescribes a fixed position and a fixed orientation while $\alpha$, $d_2$ and $d_3$ go all from maximum to minimum. As a result of this geometry transition, the $x$ axis $I_{x}$ moment of inertia goes from maximum to minimum the transition taking place between t=40 s to t=70 s (see the dashed blue line of the 
$\sqrt{I_{rel}^x}$ parameter in Fig.~\ref{fig:3RHoveringYaw} (top). 
The dimensionless $\sqrt{I_{rel}^x}$ represents a quantity related the 
$\lambda_x$ GS parameter of the controller. When $\sqrt{I_{rel}^x}$ is at its minimum $\lambda_x=0$ (lowest gain of the attitude controller). On the opposite,
when $\sqrt{I_{rel}^x}$ is at its maximum $\lambda_x=1$ (highest gain of the attitude controller).

Fig.~\ref{fig:3RHoveringYaw} (top) shows the 
3 standard deviation ($3\sigma$) estimation for the roll angle tracking error (for simulations with a zero angle reference). An estimation of the statistic is carried out based upon 1000 runs, showing in green that in the case where the formation is spread (number ``1'' in 
Fig.~\ref{fig:PitchManeuver3R} left),  
the adaptive controller is more sensitive to noise (red is covered by green) yet showing an acceptable $\pm 3\sigma$. When the formation ends its transition at t=70 s, the adaptive controller (green) employing lower gains, shows an improved response with respect to the fixed gains controller (colored in red 
on the background of the plot).
Similar responses can be seen for the pitch controller the transition being shown in   
Fig.~\ref{fig:PitchManeuver3R} (right)  
and the statistics estimation being shown in 
Fig.~\ref{fig:3RHoveringYaw} (middle).
For the yaw axis, Fig.~\ref{fig:3RHoveringYaw} (bottom) shows the comparison between the adaptive controller and the constant gains controller. With respect to the response to measurement noise, the improvement of the adaptive controller is not remarkable when doing GS
with the gains being based upon variations on the $I_{zz}$ moment of inertia of the formation.
This is due to the influence of the cross moments of inertia on the magnitude of the yaw angle estimation noise.
However, as it will be seen in the next subsection, the adaptation 
strategy based upon the $\sqrt{I_{zz}}$ pays for the yaw controller, 
considerably pays off as tracking performance is concerned with a marginal benefit in the response to noise when completely shrinking formation geometry.


\begin{figure}[t!]
    \centering
    \includegraphics[width=0.5\linewidth]{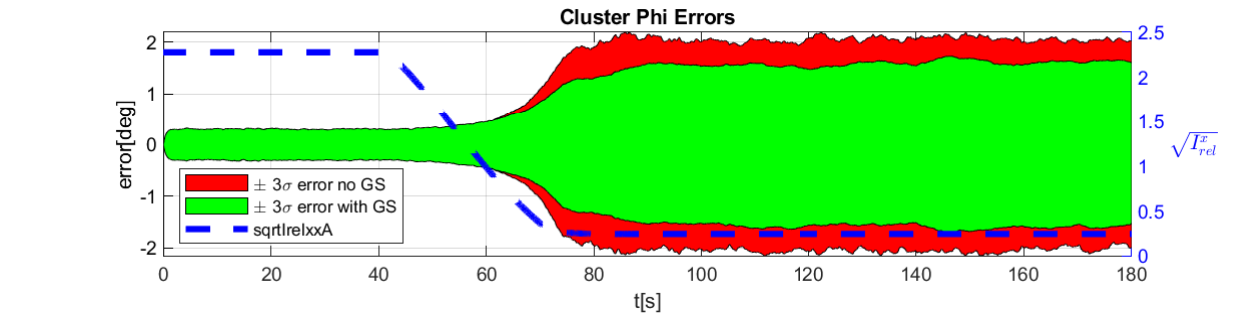}
    \includegraphics[width=0.5\linewidth,height=0.20\linewidth]{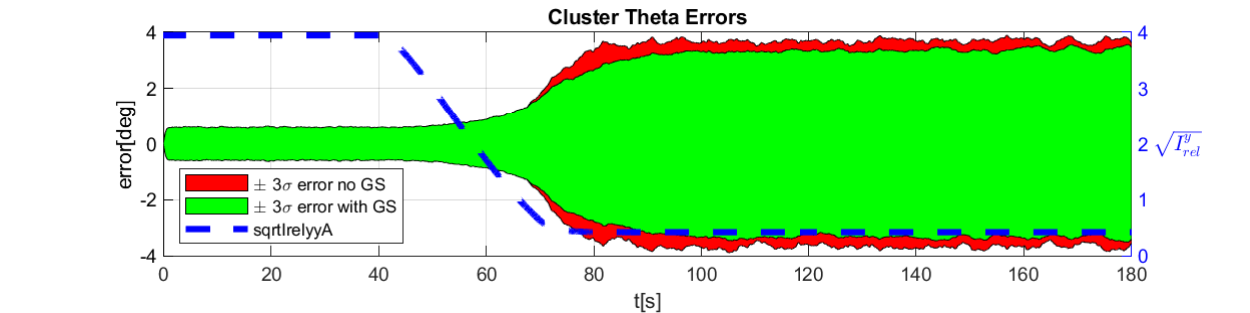}
    \includegraphics[width=0.5\linewidth,height=0.20\linewidth]{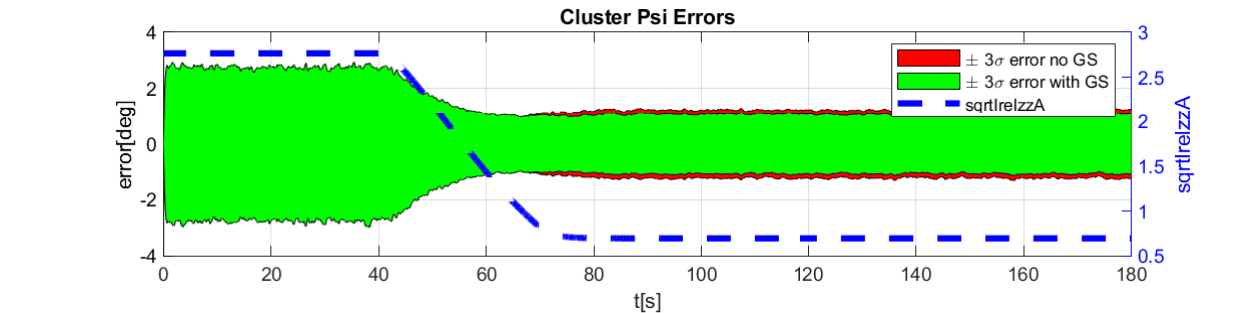}
    \caption{3R Hovering, roll, pitch \& yaw with geometry transitions.}
    \label{fig:3RHoveringYaw}    
\end{figure}


\begin{figure}[t!]
    \centering
    \includegraphics[width=0.5\linewidth]{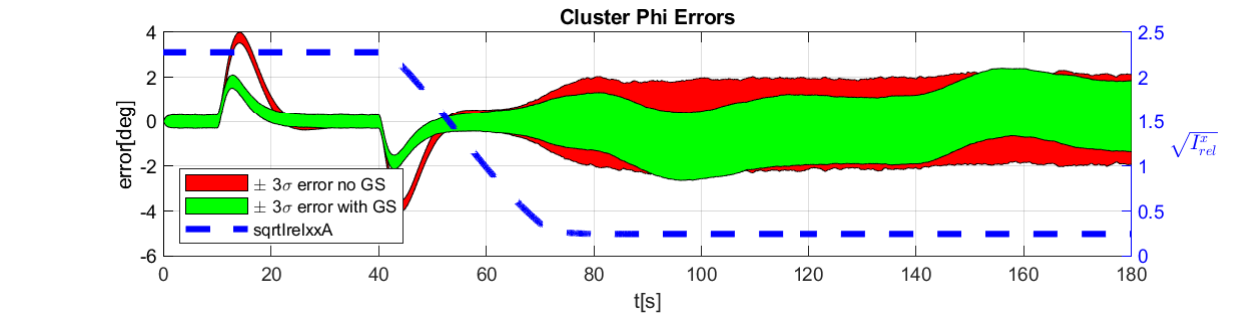}
    \includegraphics[width=0.5\linewidth]{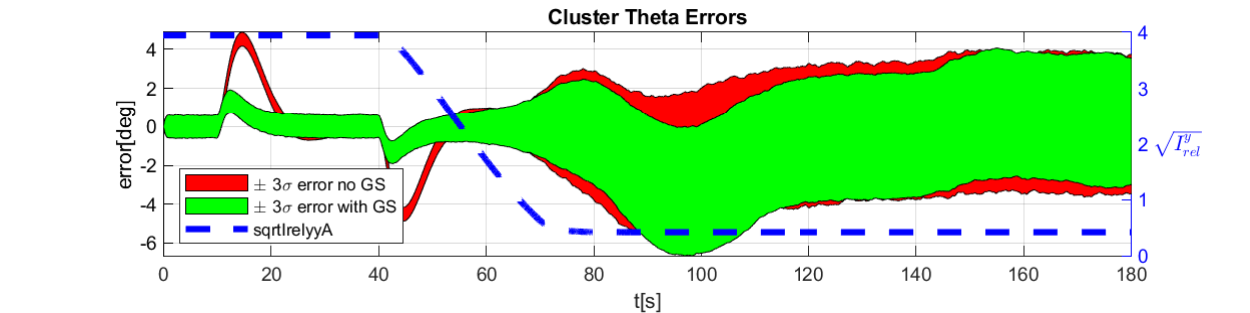}
    \centering
    \includegraphics[width=0.5\linewidth]{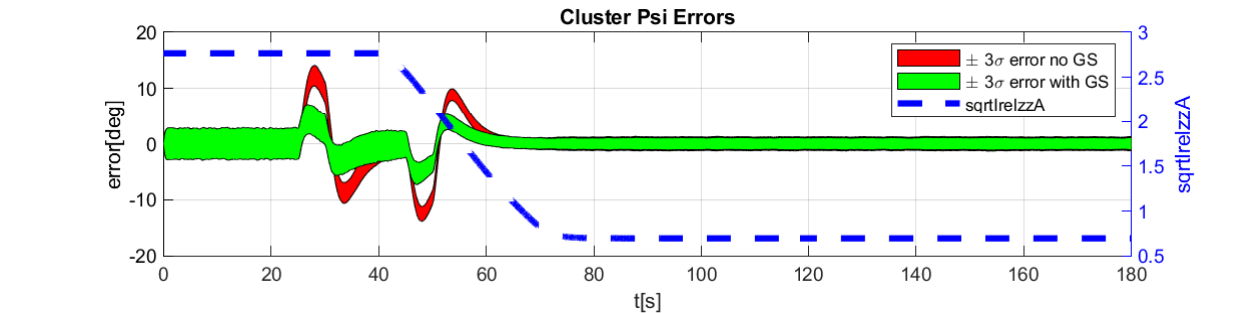}
    \includegraphics[width=0.5\linewidth]{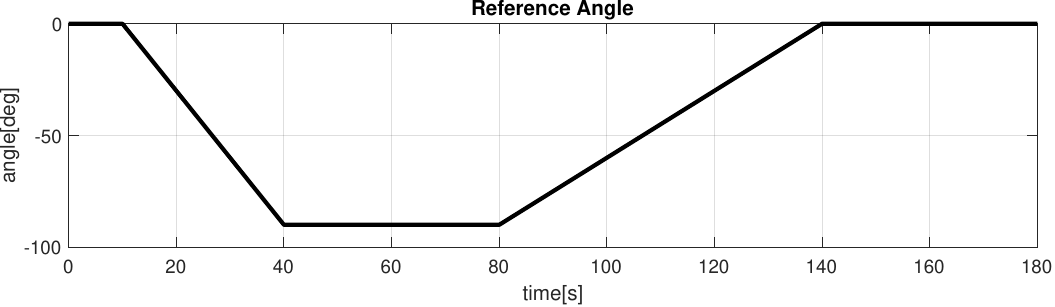}    
    \caption{3R roll, pitch \& yaw maneuvers with geometry transitions.
    At the bottom, the commanded angle profile for all maneuvers.}
    \label{fig:3RYawManeuver}
\end{figure}

\subsubsection*{3R Attitude Maneuvers}

Some of the advantages of the adaptive controller become more evident when commanding  formation maneuvers. The geometry transitions of 
Fig.~\ref{fig:PitchManeuver3R} are here repeated, adding formation rotation commands in roll, pitch and yaw (one axis for each batch of simulations) as shown in Fig.~\ref{fig:3RYawManeuver} (bottom plot).
Note in this plot, that a maneuver with a steeper slope is commanded firstly 
from t=10 s to t=40 s (30 s rotation), since simulations start with geometries rendering maximum moment of inertia in the axis of interest a situation more favorable with respect to measurement noise.
The 60s counter rotation starting at t=80 s is commanded in a softer way given the less favorable geometry with respect to measurement noise where the feedback control strategy has been one where BW is lower when using the adaptive controller. 

Note in Fig.~\ref{fig:3RYawManeuver} (first plot), the way the cloud of roll maneuvers while using the adaptive controller (in green), shows an improved average transient behavior while the acceptable noise rejection is slightly better for the red cloud (controller with constant lower gains).
When the formation transitions to a geometry of lower gains and higher noise, 
the reference signal from 80s to 140s with a softer slope, helps the adaptive controller keep the error within acceptable bounds while showing a better response to noise.
The same can be seen in Fig.~\ref{fig:3RYawManeuver} (second plot) for a pitch maneuver.
As Fig.~\ref{fig:3RYawManeuver} (third plot) is concerned, for yaw, the performance improvement of the adaptive controller with respect to the constant gains controller must be pointed out. Note the response to noise during the first 25s of the statistical time analysis, shows that a more complex adaptation scheme could be tried to improve the response of the adaptive controller to noise.
However a quick trade-off consideration suggests that the improvements with respect to transient performance are good enough to hold on to the proposed adaptation rule.

\section{EXPERIMENTAL VALIDATION}\label{Experimental}

This section is divided into two parts. In the first part, we experimentally validate the performance of the proposed control strategy for formation tracking, considering both two and three UAVs configurations. The tracking is analyzed in terms of the formation pose as well as the geometry that defines it. In the second part, we evaluate the advantages of adapting the controller to variations in the formation geometry, validating the results previously discussed in the simulation section.

The experimental platform consists of three Parrot Bebop~2 quadrotors. Real-time position and orientation measurements are obtained using an OptiTrack motion capture system installed in an indoor flight arena. The setup includes twelve infrared cameras connected to a Windows-based workstation running Motive\textregistered{}, the software responsible for managing the motion tracking. This same machine also executes a MATLAB\textregistered{} routine implementing the proposed control algorithm. Communication with the UAVs is handled by a second computer running a Linux-based ROS server, which acts as a bridge between the control system and the vehicles (see Fig.~\ref{fig:testbed}).

\begin{figure}[t!]
  \centering
  \includegraphics[width=0.5\linewidth]{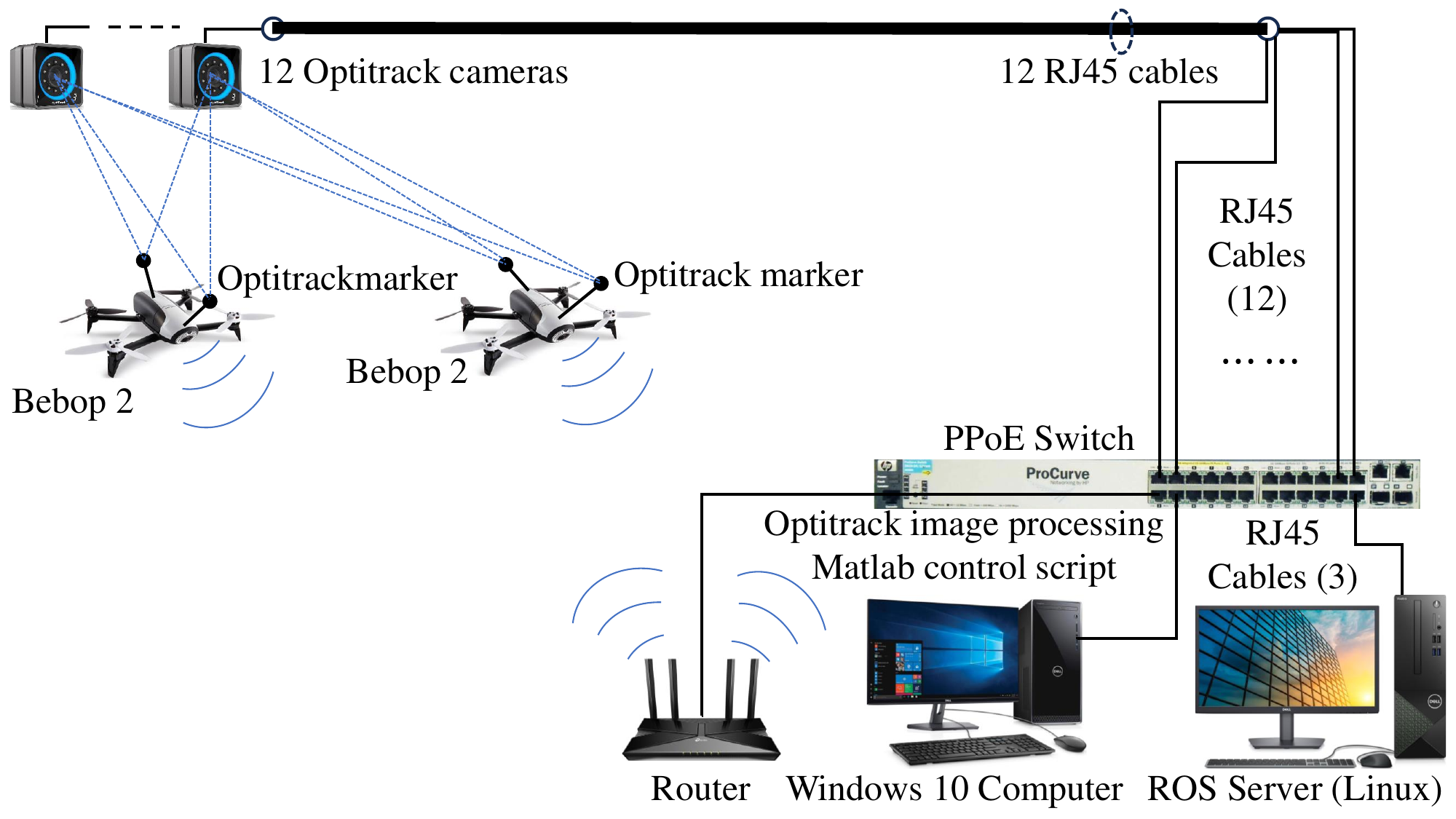}\par
  \caption{Experimental setup.}
  \label{fig:testbed}
\end{figure}

\subsection{Tracking Control Performance for UAV Formation}

Several experiments are conducted to validate the proposed control strategy. The first experiments correspond to a two-robot formation. In figure~\ref{fig:exp1}, the two-robot formation is commanded to follow a circular trajectory. Three circles can be observed: two correspond to the individual robot trajectories, while the central one corresponds to the trajectory of the formation’s center of mass (CM), showing that the inter-robot distance remains constant. In figure~\ref{fig:exp2}, the CM trajectory is depicted, confirming that it accurately tracks the commanded path, and figure~\ref{fig:exp2b} shows that both the inter-vehicle distance and the attitude of the formation remain constant.

To evaluate how the controller responds to variations in inter-robot distance and formation attitude, another experiment was conducted. Figure~\ref{fig:exp3} illustrates the responses to distance commands and to yaw and pitch changes of the formation. In this experiment, the CM position remains constant.

\begin{figure}[htbp]
    \centering
    \includegraphics[width=0.45\textwidth]{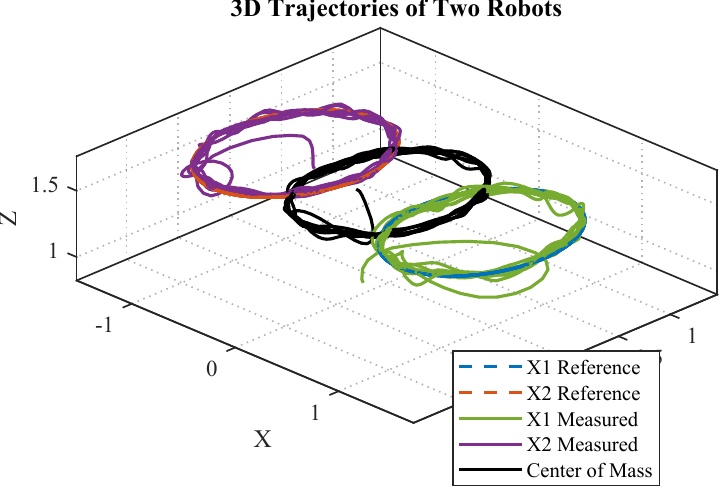}
    \caption{Circular trajectory commanded to the center of mass (CM), with the robots maintaining their relative distance. The dashed lines indicate the ideal trajectories of each robot. The violet and green lines show the actual robot trajectories, while the black line corresponds to the real CM trajectory.}
    \label{fig:exp1}
\end{figure}

\begin{figure}[htbp]
    \centering
    \includegraphics[width=0.45\textwidth]{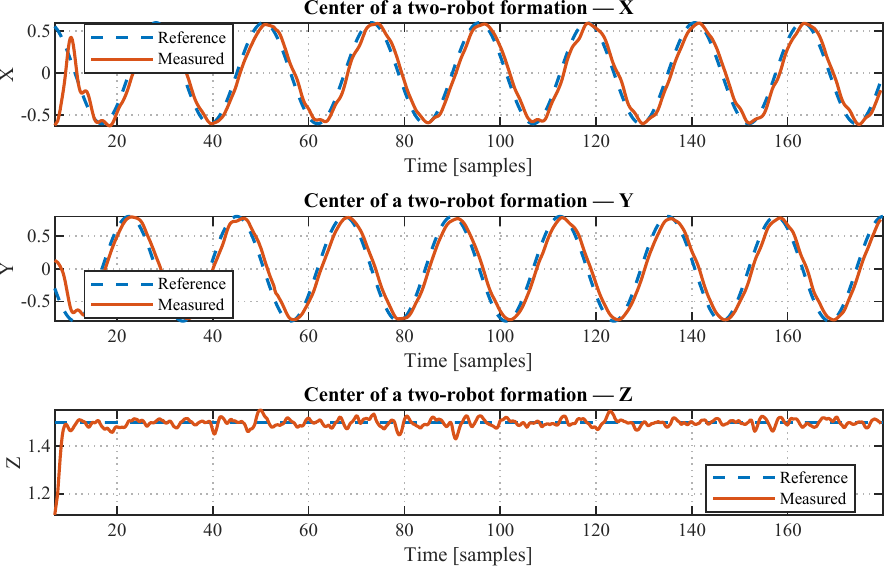}
    \caption{The solid line represents the position of the center of mass of the formation, while the dashed line indicates the desired position.}
    \label{fig:exp2}
\end{figure}

\begin{figure}[htbp]
    \centering
    \includegraphics[width=0.45\textwidth]{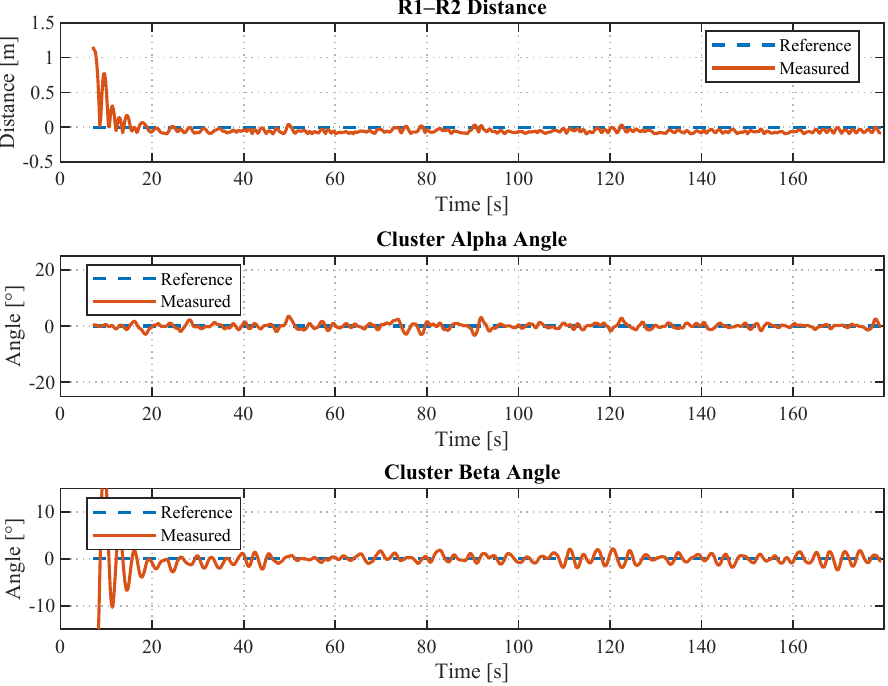}
    \caption{The solid line represents the inter-vehicle distances and the attitude of the formation, while the dashed line indicates the desired position.}
    \label{fig:exp2b}
\end{figure}

\begin{figure}[htbp]
    \centering
    \includegraphics[width=0.45\textwidth]{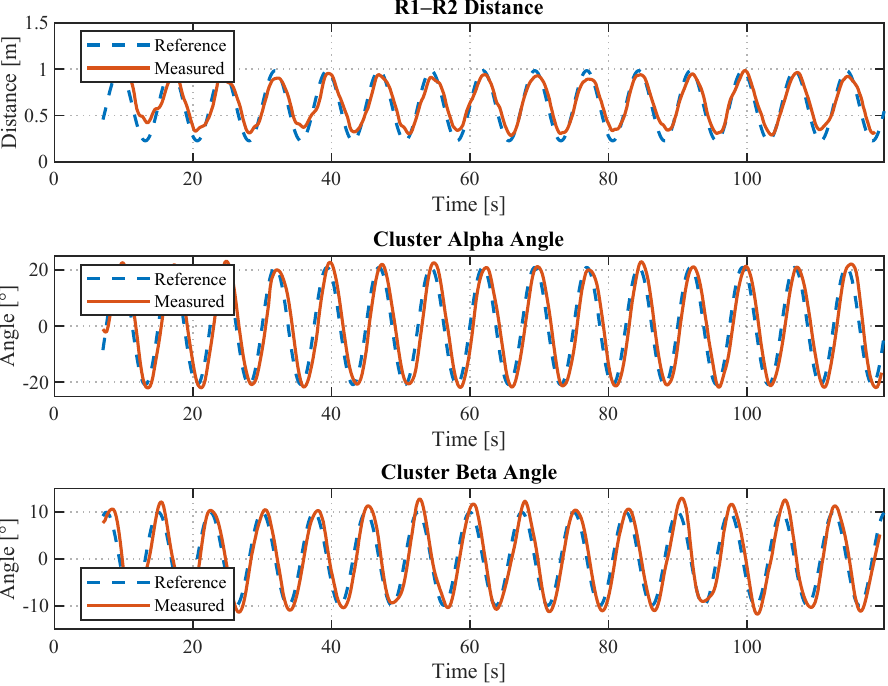}
    \caption{From top to bottom: inter-robot distance, formation yaw angle, and formation pitch angle. Dashed lines represent the reference commands, while solid lines show the measured variables.}
    \label{fig:exp3}
\end{figure}

The following experiments involve formations of three robots. Figure~\ref{fig:formation_center} shows that the center of a three-UAV formation successfully tracks a predefined trajectory. Due to space limitations in the arena, the commanded trajectory in this case is simpler: a square path with variations in both forward direction and altitude, as shown in figure~\ref{fig:3Robot3Dtrajectory}. Figure~\ref{fig:formation_geometry} presents the parameters that describe the formation geometry, namely the inter-vehicle distances and angle $\alpha$, which remain close to their desired constant values during the maneuver. 

Another experiment was conducted to analyze how the formation responds to changes in its geometry. Figure~\ref{fig:adaptive_geometry} illustrates a maneuver in which the formation adapts its shape by modifying the inter-vehicle distances, while simultaneously maintaining the center of mass, as illustrated in figure~\ref{fig:adaptive_geometry_positions}.
 
\begin{figure}[t!]
  \centering
  \includegraphics[width=0.5\linewidth]{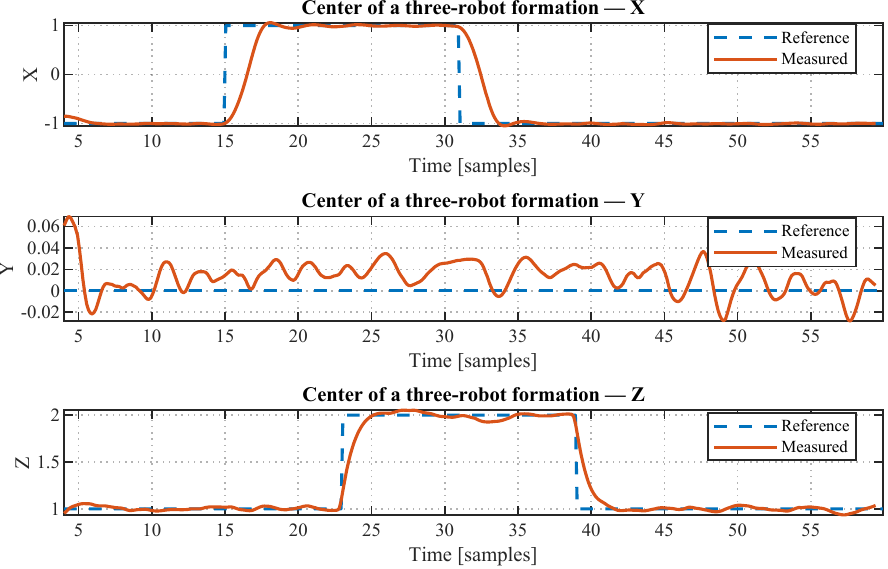}\par
  \caption{Trajectory tracking of the formation center for three UAVs.}  
  \label{fig:formation_center}
\end{figure}

\begin{figure}[t!]
  \centering
  \includegraphics[width=0.5\linewidth]{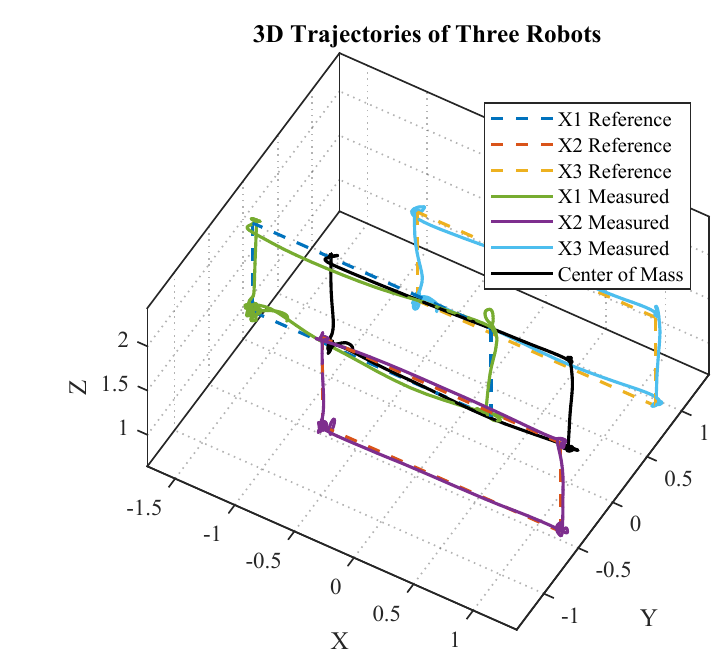}\par
  \caption{Trajectories of the three UAVs showing the formation tracking performance.}
  \label{fig:3Robot3Dtrajectory}
\end{figure}

\begin{figure}[t!]
  \centering
  \includegraphics[width=0.5\linewidth]{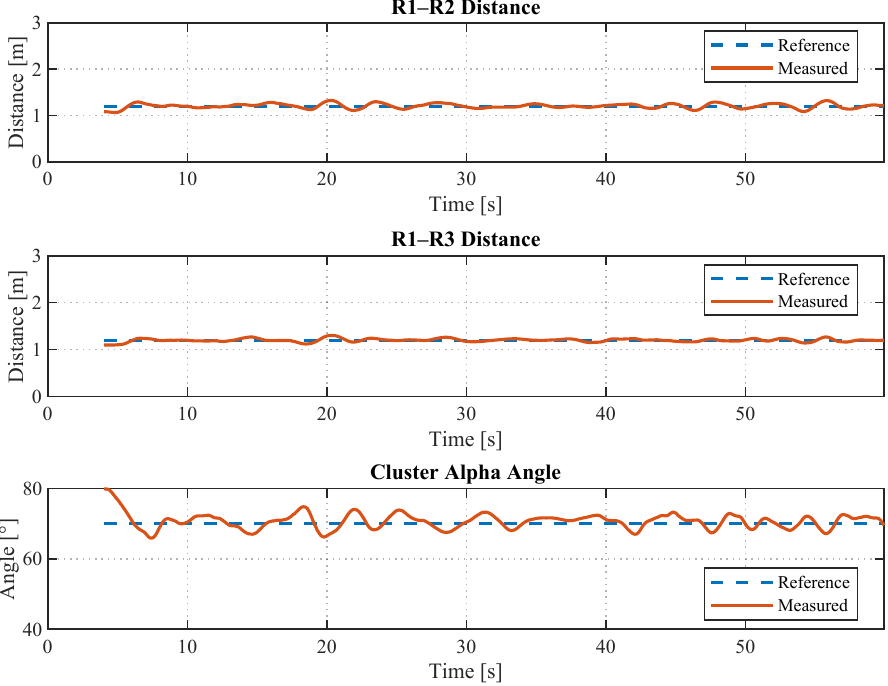}\par
  \caption{Relative distances and formation angle among three UAVs.}
  \label{fig:formation_geometry}
\end{figure}

\begin{figure}[t!]
  \centering
  \includegraphics[width=0.5\linewidth]{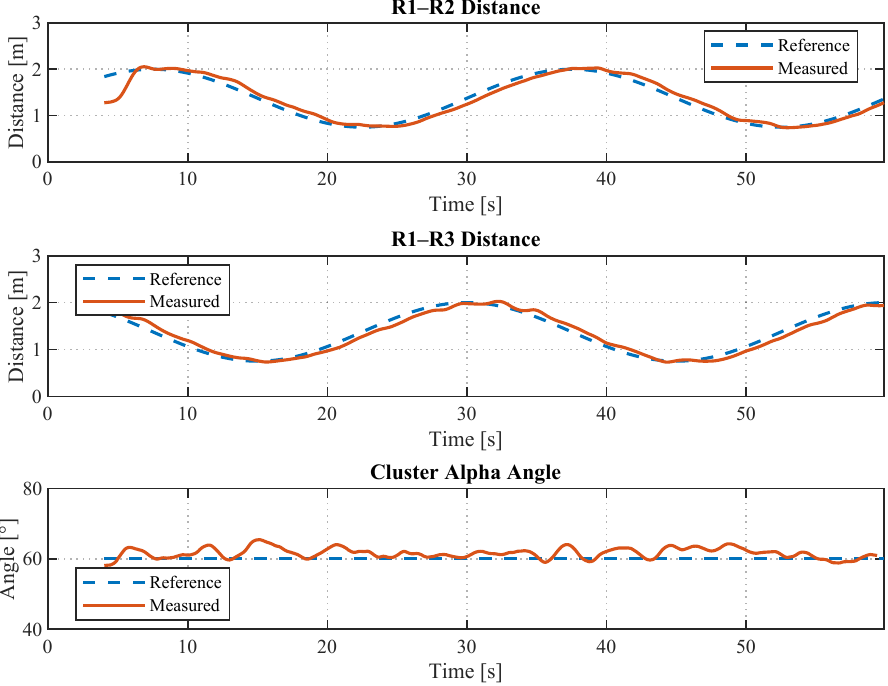}\par
  \caption{Trajectories with varying inter-vehicle distances.}
  \label{fig:adaptive_geometry}
\end{figure}

\begin{figure}[t!]
  \centering
  \includegraphics[width=0.5\linewidth]{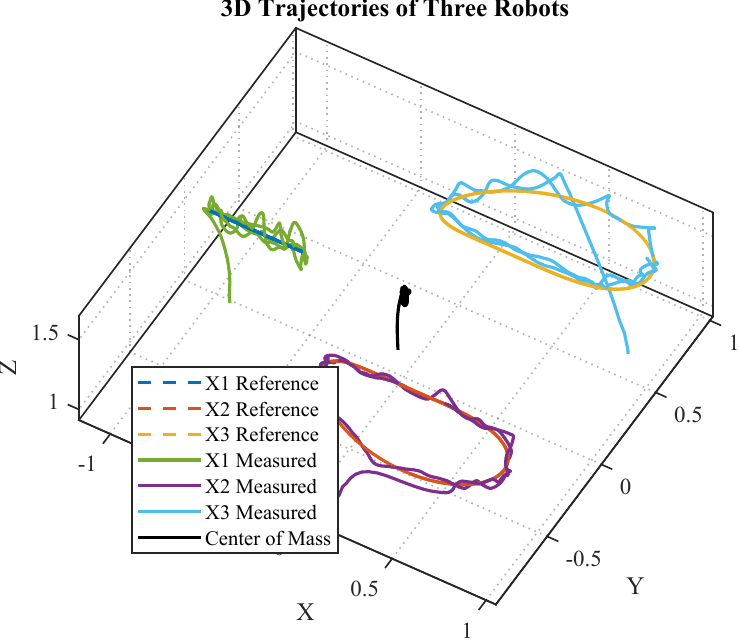}\par
  \caption{Trajectories with varying inter-vehicle distances, while maintaining the center of mass of the formation.}
  \label{fig:adaptive_geometry_positions}
\end{figure}

Another experiment was conducted to evaluate the attitude tracking performance. In one of these experiments, a roll angle reference trajectory was sent to the formation while maintaining its geometry. Figure~\ref{fig:attitude_cluster} illustrates how the formation follows the commanded trajectory. In figure~\ref{fig:attitude_cluster_pitch}, the resulting trajectory of the robots can be observed. It also shows that the center of mass, as well as the geometry of the formation, remains constant.

\begin{figure}[t!]
  \centering
  \includegraphics[width=0.5\linewidth]{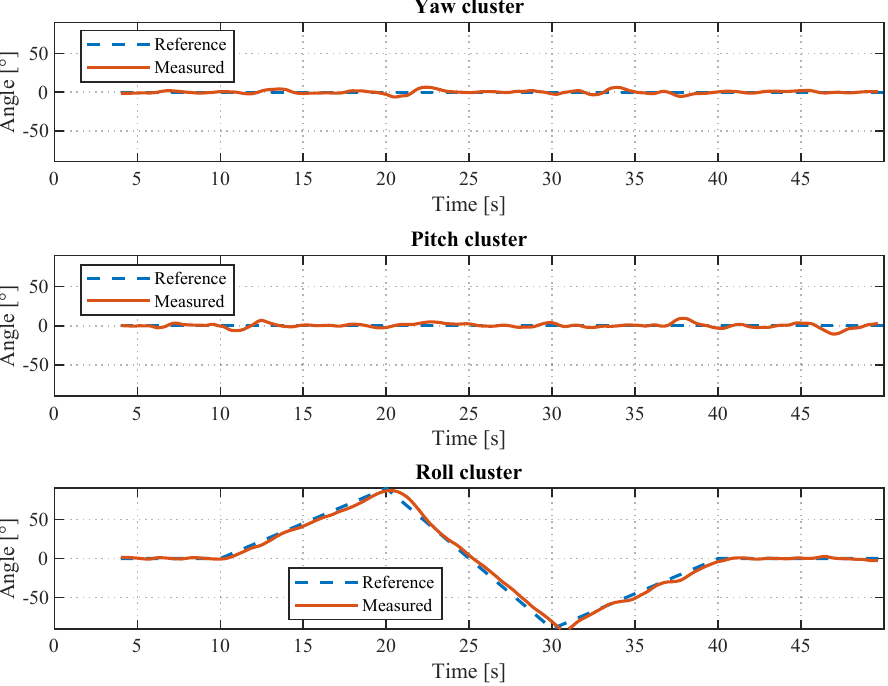}\par
  \caption{Attitude tracking of the formation in yaw, pitch, and roll. Dashed lines denote reference commands, while solid lines show the measured responses.}
  \label{fig:attitude_cluster}
\end{figure}

\begin{figure}[t!]
  \centering
  \includegraphics[width=0.5\linewidth]{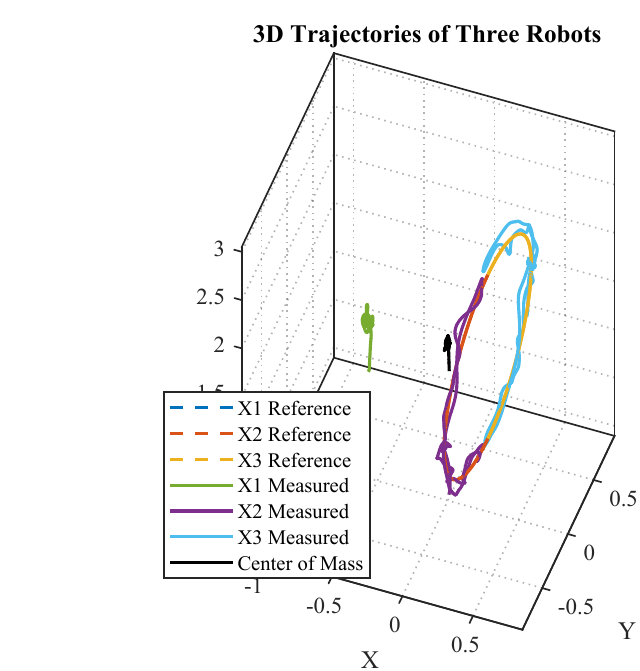}\par
  \caption{Attitude tracking of the formation in yaw, pitch, and roll. Dashed lines denote reference commands, while solid lines show the measured responses.}
  \label{fig:attitude_cluster_pitch}
\end{figure}

\subsection{Adaptation to Geometric Variations}

In this section, we experimentally analyze how the control performance improves when the controller gains are adapted according to the geometry of the formation. Since the space available for testing is very limited, the range of possible formation shape variations is also restricted, making the impact of gain adaptation difficult to appreciate. However, thanks to the high accuracy of the motion capture system, the improvement in control performance due to gain adaptation, as previously proposed, can still be observed. To demonstrate this, several flight tests were carried out by modifying the inter-vehicle distances, while a yaw reference was also commanded to the formation.

Figure~\ref{fig:Adapt3Dtray} shows the trajectory followed by the vehicles during the experiment. 
These trajectories can be observed in figures~\ref{fig:AdaptDistances} and \ref{fig:AdaptYaw} for the inter-vehicle distances and formation orientation, respectively. Between 80 and 180 seconds, yaw commands and inter-vehicle distance variations (between 120 and 155 seconds) were applied while keeping the controller gains fixed, i.e., without adapting them to changes in the formation geometry. At second 180, an adaptive controller—whose gains adjust according to the geometry—was employed, and the same yaw commands and inter-vehicle distance variations were applied, as shown in figures~\ref{fig:AdaptDistances} and \ref{fig:AdaptYaw}.

In particular, figure~\ref{fig:AdaptCM} highlights the improvement achieved by introducing adaptive control. This can be seen between 120 and 160 seconds, where the fixed-gain controller exhibits lower performance compared to the interval between 235 and 280 seconds, during which the formation was subjected to the same trajectory but with gain adaptation enabled. This can also be noted in figure~\ref{fig:AdaptDistances}.

\begin{figure}[t!]
  \centering
  \includegraphics[width=0.5\linewidth]{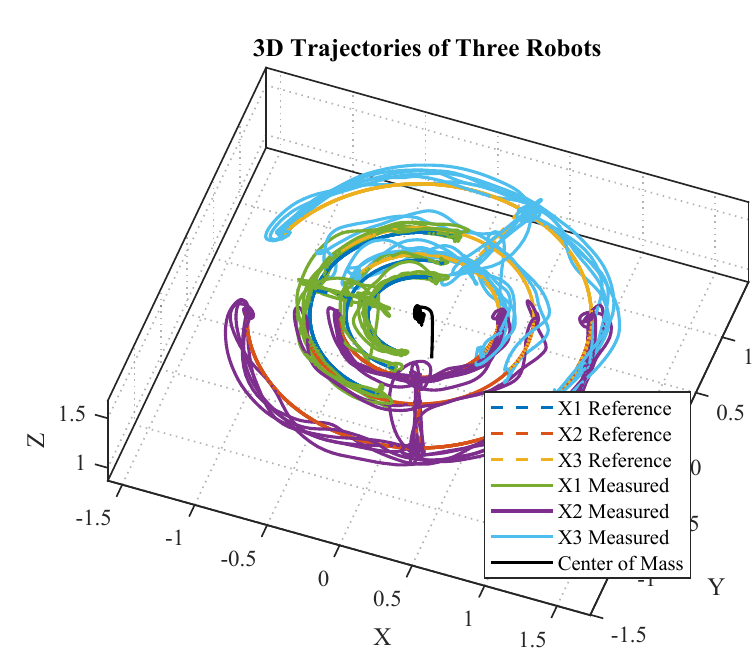}\par
  \caption{Three-dimensional trajectory of the formation during the experiment. 
  The vehicles follow a yaw reference while the inter-vehicle distances are varied.}
  \label{fig:Adapt3Dtray}
\end{figure}

\begin{figure}[t!]
  \centering
  \includegraphics[width=0.5\linewidth]{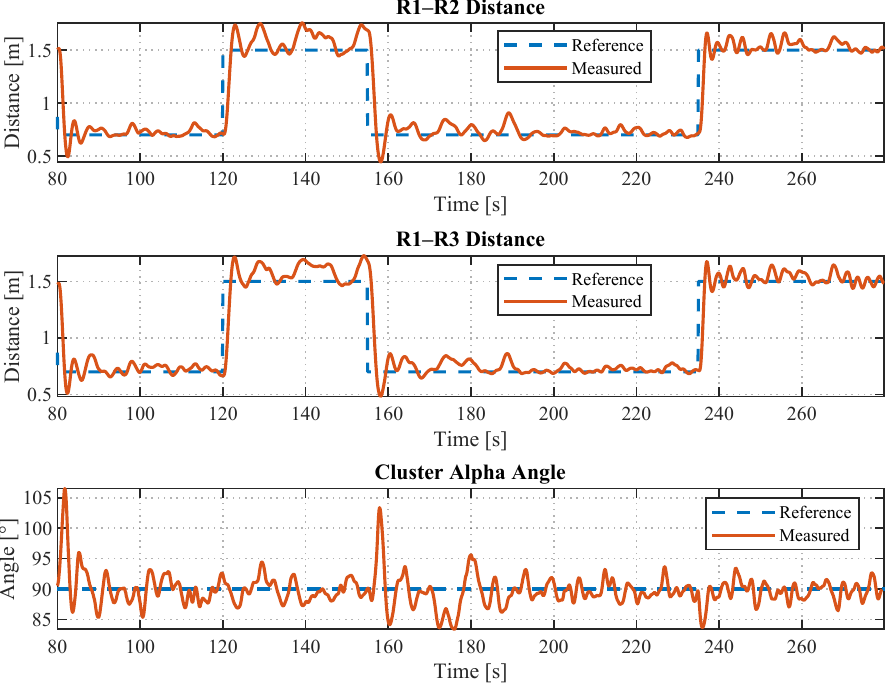}\par
  \caption{Inter-vehicle distance variations and yaw reference commands applied to the formation during the experiment. Dashed lines denote the reference signals, while solid lines show the measured responses.}
  \label{fig:AdaptDistances}
\end{figure}

\begin{figure}[t!]
  \centering
  \includegraphics[width=0.5\linewidth]{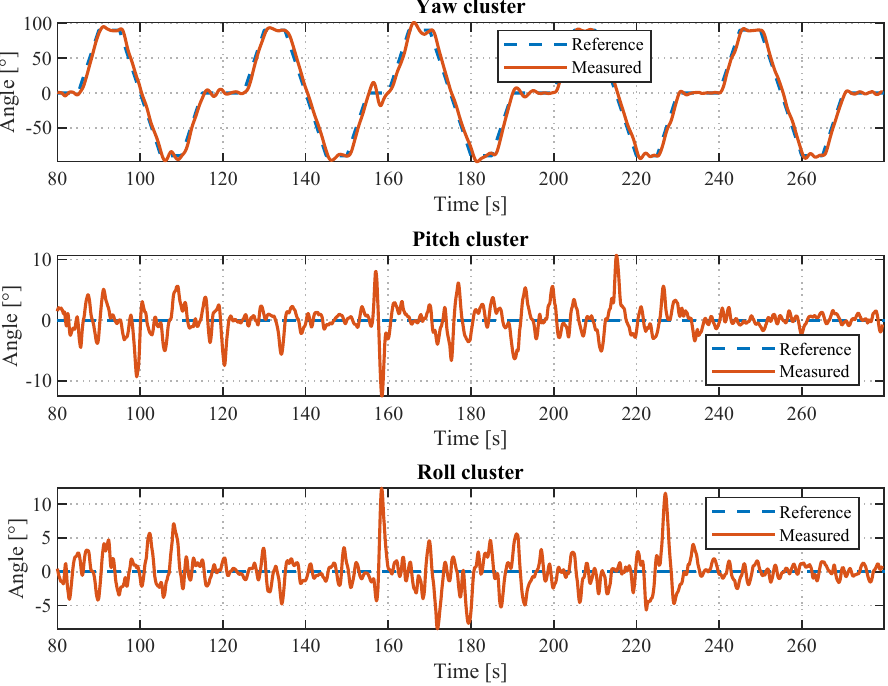}\par
  \caption{Yaw attitude tracking of the formation. Dashed lines denote the reference commands, while solid lines represent the measured responses.}
  \label{fig:AdaptYaw}
\end{figure}

\begin{figure}[t!]
  \centering
  \includegraphics[width=0.5\linewidth]{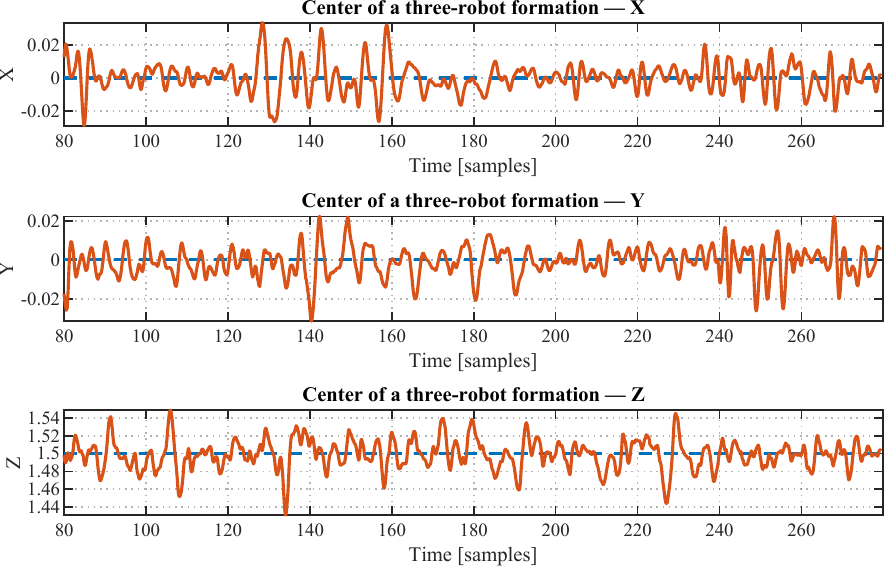}\par
  \caption{Center-of-mass tracking of the three-robot formation. Dashed lines denote the reference trajectory, while solid lines show the measured response.}
  \label{fig:AdaptCM}
\end{figure}

\section{CONCLUSION}

This paper presented a formation control strategy for UAVs based on a dual quaternion representation. By exploiting the compact and geometrically consistent properties of dual quaternions, the proposed controller unifies translational and rotational dynamics in a single framework, thereby simplifying control law design. A key contribution of this work is the introduction of adaptive controller gains that depend on formation geometry, extending previous fixed-gain approaches.

Simulation and experimental results with two- and three-UAV formations confirm the effectiveness of the proposed controller. The results show accurate tracking of commanded trajectories and attitude angles while preserving the formation geometry. Moreover, the adaptive gain strategy significantly improves performance, especially during inter-vehicle distance variations and formation reconfigurations. The close agreement between simulation and experimental results highlights both the robustness and practical applicability of the proposed dual quaternion-based control strategy across a wide range of operating conditions.

\bibliographystyle{plain} \bibliography{referencias}

\end{document}